\documentclass[journal]{IEEEtran}
\usepackage{amsmath,amssymb,amsfonts}
\usepackage{cite}
\usepackage{algorithm}
\usepackage{algorithmic}
\usepackage{graphicx}
\usepackage{caption}
\usepackage{lipsum}
\usepackage{amsthm}
\usepackage{threeparttable}
\usepackage{color}
\usepackage{amssymb}
\usepackage{tabularx}
\usepackage{enumitem}
\usepackage{booktabs}
\usepackage{cuted}
\usepackage{flushend}
\newcommand{\Exp}{\mathop{\mathbb E}\displaylimits}

\ifCLASSOPTIONcompsoc
  \usepackage[caption=false,font=normalsize,labelfont=sf,textfont=sf]{subfig}
\else
  \usepackage[caption=false,font=footnotesize]{subfig}
\fi
\newtheorem{theorem}{Theorem}

\newtheorem{lemma}{Lemma}
\newtheorem{condition}{Condition}
\newtheorem{regular_conditions}{Regular Conditions}
\newcommand{\lc}[1]{\left \{ #1 \right.}
\newcommand{\rc}[1]{\left. #1 \right \}}
\newcommand{\lrc}[1]{\left \{ #1 \right \}}

\renewcommand{\rVert}[1]{\left. #1 \right \Vert}
\newcommand{\lrVert}[1]{\left \Vert #1 \right \Vert}

\newcommand{\Ed}[2]{\operatorname*{\mathbb{E}}_{#1} \lrc{#2}}
\newcommand{\Edl}[2]{\operatorname*{\mathbb{E}}_{#1} \lc{#2}}

\begin{document}
\title{Mixed Policy Gradient: off-policy reinforcement learning driven jointly by data and model}
\author{Yang Guan\textsuperscript{1}, Jingliang Duan\textsuperscript{1}, Shengbo Eben Li\textsuperscript{1}, Jie Li\textsuperscript{1}, Jianyu Chen\textsuperscript{2}, Bo Cheng\textsuperscript{1}
\thanks{
This work has been submitted to the IEEE for possible publication. Copyright may be transferred without notice, after which this version may no longer be accessible.

This work was supported by International Science \& Technology Cooperation Program of China under 2019YFE0100200, Tsinghua University-Toyota Joint Research Center for AI Technology of Automated Vehicle, and NSF China with 51575293, and U20A20334.}
\thanks{\textsuperscript{1}School of Vehicle and Mobility, Tsinghua University, Beijing, 100084, China. \textsuperscript{2}Institute for Interdisciplinary Information Sciences, Tsinghua University, Beijing, 100084, China. All correspondence should be sent to S. Eben Li. $<$lisb04@gmail.com$>$.}
}


\maketitle

\begin{abstract}
Reinforcement learning (RL) shows great potential in sequential decision-making. At present, mainstream RL algorithms are data-driven, which usually yield better asymptotic performance but much slower convergence compared with model-driven methods. This paper proposes mixed policy gradient (MPG) algorithm, which fuses the empirical data and the transition model in policy gradient (PG) to accelerate convergence without performance degradation. Formally, MPG is constructed as a weighted average of the data-driven and model-driven PGs, where the former is the derivative of the learned Q-value function, and the latter is that of the model-predictive return. To guide the weight design, we analyze and compare the upper bound of each PG error. Relying on that, a rule-based method is employed to heuristically adjust the weights. In particular, to get a better PG, the weight of the data-driven PG is designed to grow along the learning process while the other to decrease. Simulation results show that the MPG method achieves the best asymptotic performance and convergence speed compared with other baseline algorithms.
\end{abstract}


\section{Introduction}
Reinforcement learning (RL) algorithms have been applied in a wide variety of challenging domains and achieved good performance, ranging from games to robotic control \cite{ mnih2015human,guan2020centralized, vinyals2019grandmaster,duan2023optimization,duan2023optimization2,duan2021encoding}. Policy gradient (PG) is one of the main categories of RL and has drawn great attention due to its suitability for continuous action space and deterministic policy. PG methods can be categorized into data-driven methods and model-driven methods. The former learns a policy purely from empirical data generated by interaction between the agent and a environment or high-fidelity simulators, whereas the latter learns through backpropagation through time (BPTT) with the help of transition models. In addition, PG methods can be further categorized into on-policy and off-policy. On-policy methods solely utilize data generated by the current policy to construct PG, whereas off-policy methods can use data from any policy. Therefore, samples can be reused, leading to improved data efficiency. Below, we will introduce existing methods using the classification of data-driven and model-driven, and unless otherwise specified, we will primarily focus on off-policy PG methods. 
    
Data-driven PG methods only require data to learn a control policy. The PG theorem was first proposed by Sutton \emph{et al.} (2000), who showed that the action-value (i.e., Q-value) function could be used to calculate the PG of stochastic policies \cite{sutton2000policy}. The deterministic version of the PG (DPG) theorem was discovered by Silver \emph{et al.} (2014). DPG also needs a Q-value to compute its gradient with respect to the action. In data-driven PG, imprecise value function may lead to poor PG, thereby negatively affecting policy performance.
Therefore, several tricks have been proposed to improve the estimation accuracy of value functions. Inspired by Deep Q Networks \cite{mnih2015human}, Lillicrap \emph{et al.} (2015) used a stochastic policy to explore in the environment to learn a more accurate Q-value function \cite{lillicrap2015continuous}. Besides, they introduced a target Q-value function to stabilize the learning process. Schulman \emph{et al.} (2016) estimated the value function in the PG using an exponentially-weighted estimator, substantially reducing the bias introduced by the learned value function \cite{schulman2015high}. Combining with techniques that guarantee monotonic improvement (such as TRPO and PPO\cite{schulman2015trust, schulman2017proximal}), they yielded strong empirical results on highly challenging 3D locomotion tasks. Fujimoto \emph{et al.} (2018) took the minimum value between a pair of Q-value functions to avoid overestimation \cite{fujimoto2018addressing}. 
Duan \emph{et al.} proposed to learn a continuous return distribution to address value estimation errors, achieving state-of-the-art performance on the suite of MuJoCo continuous control tasks \cite{duan2021distributional}. 


Model-driven PG methods require a differentiable model which can be either given as a prior or learned from empirical data\cite{yang2021hamiltonian}. With the model, the analytical PG can be directly calculated by backpropagation of rewards along a trajectory.
In this way, the influence of the estimated value function on PG accuracy can be largely reduced. Therefore, leveraging model can speed up the convergence and reduce the sample complexity. A pioneer work is PILCO proposed by Deisenroth and Rasmussen (2011), in which they learned a probabilistic model, and computed an analytical PG over a finite-horizon reward sum to learn a deterministic policy. Results showed that PILCO can facilitate learning from scratch in only a few trials \cite{deisenroth2011pilco}. Heess \emph{et al.} (2015) extended BPTT to infinite-horizon and stochastic policy using a learned value function and re-parameterization trick \cite{heess2015learning}. However, there are three main disadvantages of model-driven PG methods. The first is the numerical instability of long chain nonlinear computations, which leads to random and meaningless PG. To address this problem, Parmas \emph{et al.} (2019) introduced an additional likelihood ratio PG estimator and automatically gave greater weight to estimators with lower variance \cite{parmas2019pipps}. Second, the BPTT PG is sensitive to model errors, especially for a long prediction horizon \cite{mu2020mixed}. A common remedy for this is to truncate the model prediction by a value function \cite{feinberg2018model,buckman2018sample,janner2019trust}. Finally, the gradient computing time of BPTT grows linearly with the length of predictive horizon, which eliminates the benefits of fast convergence.

Recent years, some studies tried to use both the data and model for better overall performance. A straightforward and mostly used idea is to fuse them on the model perspective, i.e., first improving an existing model by data and then learning from it. Shi \emph{et al.} (2019) learned the uncertain environment disturbances based on the drone analytic model and realized stable landing by a feedback controller\cite{shi2019neural}. Similarly, Mu \emph{et al.} (2020) fitted the noise term of the model, and then learned a policy by the BPTT PG\cite{mu2020mixed}. However, this kind of methods suffer overfitting issues. The model accuracy is highly associated to the data distribution, out of which the performance can be arbitrary bad. To cope with this, some methods learn an ensemble model and select to use by certain rules \cite{janner2019trust,kurutach2018model}; some build the probabilistic model to consider model uncertainty \cite{deisenroth2011pilco}; and some directly restrict high order derivatives of model to prevent its overfitting\cite{shi2019neural}.

Since all these disadvantages of existing methods, we propose to employ both the empirical data and transition model in the PG level to get advantages of both fast convergence and good asymptotic performance. Compared with some existing studies, the main contributions emphasize in three parts.

1) We propose a mixed policy gradient (MPG) reinforcement learning method which combines the data-driven PG and the model-driven PG by weighted average to get benefit from fast convergence speed of model and better asymptotic performance of data.

2) We derive the theoretical upper bound of the data-driven and the model-driven PG bias, i.e., the distance with the ground-truth PG. Relying on their relative magnitude, we design a rule-based method to adaptively adjust the data and model PG weights during policy update for best performance.

3) We build an asynchronous learning architecture to address the poor computing efficiency suffered by BPTT, in which multiple learners are employed to compute the MPG in parallel to enhance update throughput.
\section{Notations and Preliminaries}
\subsection{Notations}
\begin{enumerate}
    \item $s,a,r,s'$: state, action, reward, next state
    \item $\mathcal{S},\mathcal{A}$: state space, action space
    \item $p(\cdot|s,a),p(s,a)$: stochastic/deterministic dynamics
    \item $f(\cdot|s,a),f(s,a)$: stochastic/deterministic model
    \item $r(s,a)$: reward function
    \item $\pi(s,a),\pi_{\theta}(s,a)$: policy, policy network
    \item $d^0(s)$: initial state distribution
    \item $\gamma$: discount factor
    \item $\tau(s_t,a_t)$: trajectory that starts from $(s_t, a_t)$ and follows $\pi$ and $p$, i.e., $(s_t,a_t,r_t,s_{t+1},a_{t+1},r_{t+1},\dots)$
    \item $\hat{\tau}(s_t,a_t)$: trajectory that starts from $(s_t, a_t)$ and follows $\pi$ and $f$, i.e., $(s_t,a_t,\hat{r}_t,\hat{s}_{t+1},\hat{a}_{t+1},\hat{r}_{t+1},\dots)$
    \item $\tau(s_t)$: trajectory that starts from $s_t$ and follows $\pi$ and $p$, i.e., $\tau(s_t,\pi(s_t))$
    \item $\hat{\tau}(s_t)$: trajectory that starts from $s_t$ and follows $\pi$ and $f$, i.e., $\hat{\tau}(s_t,\pi(s_t))$
    \item $q^{\pi}(s,a),Q_{\omega}(s,a)$: value, value network
    \item $\rho^{\pi}(s),d^{\pi}(s)$: discount visiting frequency, stationary distribution
\end{enumerate}

\subsection{Preliminaries}
We consider a Markov decision process (MDP) defined by the tuple $(\mathcal{S}, \mathcal{A}, p, r, \gamma)$. Define the Q-value function as:
\begin{equation}
\nonumber
q^{\pi}(s,a) =\Ed{\tau(s,a)}{\sum_{l=t}^{\infty} \gamma^{l-t} r_l}.
\end{equation}
We approximate Q-value by a network $Q_{\omega}(s, a),$ and approximate policy by a network $\pi_{\theta}(s)$.
The goal of RL is to optimize the policy network for maximizing the expected sum of discounted rewards:
\begin{equation}\label{eq.objective_function}
J(\theta) = \Ed{s\sim d^0}{q^{\pi_{\theta}}(s, \pi_{\theta}(s))}.
\end{equation}
From the theory of DPG \cite{silver2014deterministic}, its gradient is
\begin{equation}\label{eq.pg}
\nabla_{\theta} J(\theta) = \Ed{s_t\sim\rho^{\pi_{\theta}}}{\nabla_{\theta}\pi_{\theta}(s_t)\nabla_{a_t} q^{\pi_{\theta}}(s_t, a_t)\bigg|_{a_t = \pi_{\theta}(s_t)}}.
\end{equation}
where $\rho^{\pi_{\theta}}$ is the discount visiting frequency\cite{guan2021direct}.
For simplicity, we will omit the arguments of the derivative function, and use $s_t\sim\rho^{\pi_{\theta}}$ unless stated.

The data-driven PG is calculated by substituting $q^{\pi_{\theta}}$ in equation \eqref{eq.pg} with its estimate $Q_w$:
\begin{equation}\label{eq.pg_data}
\nabla_{\theta} J^{\text{Data}}(\theta) = \Ed{s_t}{\nabla_{\theta}\pi_{\theta}\nabla_{a_t} Q_{\omega}}.
\end{equation}

By contrast, the model-driven PG is constructed by first approximating the objective \eqref{eq.objective_function} using model, and then taking the gradient of it:
\begin{equation}\label{eq.pg_model_infty}
\nabla_{\theta} J^{\text{Model}}(\theta) = \Ed{s_t}{\nabla_{\theta} \Ed{\hat{\tau}(s_t)}{\sum_{l=t}^{\infty} \gamma^{l-t} \hat{r}_l}}.
\end{equation}
In practice, the infinite sum is usually truncated by a learned value function.
\section{Mixed Policy Gradient}
Different from current methods, this section introduces a novel gradient construction method: weighted average of data-driven and model-driven PGs. The core idea is illustrated by the following formula:
\begin{equation}
\label{eq.MPG_form}
    \nabla_{\theta}J^{\text{Mixed}}(\theta) = w_{\text{data}} \nabla_{\theta}J^{\text{Data}}(\theta) +w_{\text{model}}\nabla_{\theta}J^{\text{Model}}(\theta)
\end{equation}
where $w_{\text{data}}$ and $w_{\text{model}}$ represent the weights assigned to the data-driven and model-driven components, respectively. While the concept is straightforward, the challenge lies in the weight design. It requires a more in-depth investigation into the gradients of data-driven and model-driven approaches.

Therefore, this section first examines the PG biases of data-driven and model-driven approaches and subsequently derives the mathematical form of MPG. A rule-based weight design method is then presented in the following section.

\subsection{Data/Model policy gradient unification}
To study the data-driven and model-driven PGs, we first unify the data and model PGs by replacing the value function in \eqref{eq.pg} with $n$-step Bellman recursion:
\begin{equation}
\label{eq.unified_PG}
\begin{aligned}
&\nabla_{\theta} J_n(\theta) =\\
&\Ed{s_t}{\nabla_{\theta}\pi_{\theta}\nabla_{a_t}\Ed{\hat{\tau}(s_t)}{\sum_{l=t}^{n-1+t} \gamma^{l-t} \hat{r}_l + \gamma^{n} \hat{Q}_{\omega, t+n}}}
\end{aligned}
\end{equation}
where $n$ is the model rollout steps and  $\hat{Q}_{\omega,t+n}$ is short for $Q_{\omega}(\hat{s}_{t+n},\hat{a}_{t+n})$. We call equation \eqref{eq.unified_PG} the unified PG. Intuitively, the unified PG is determined by both the data and model, where the data influences it by the value network $Q_{\omega}$ and the model by the differentiable trajectory $\hat{\tau}(s_t)$. Its relation with the other PGs can be revealed by the following theorem.

\begin{theorem}\label{theorem.model_pg}
The data-driven PG has the following relation with the unified PG:
\begin{equation*}
    \nabla_{\theta}J^{\text{Data}}(\theta) = \nabla_{\theta}J_0(\theta),
\end{equation*}
and if $f=p$ and the $\rho^{\pi_{\theta}}$ is the stationary distribution, the model-driven PG and the unified PG has the following equivalence:
\begin{equation}
\nonumber
\nabla_{\theta} J^{\text{Model}}(\theta) =\frac{1}{1-\gamma}\nabla_{\theta}J_{\infty}(\theta).
\end{equation}
\end{theorem}
\begin{proof}
See Appendix \ref{appendix.proof_of_model_pg}.
\end{proof}

Theorem \ref{theorem.model_pg} shows that both the data-driven and model-driven PGs are special cases of the unified PG, which explains the behaviors of these methods:
From \eqref{eq.unified_PG}, it is clear that the accuracy of the unified PG depends on model predictive errors and Q-value estimate errors. In particular, as the predictive horizon $n$ increases, the influence of estimate errors decays, while model errors accumulate continuously. As a result, the data-driven PG heavily relies on the Q-value estimate, whose inaccuracy especially in early training stage makes the PG actually random and thus slow convergence. Specifically, the convergence speed means the needed policy update number to reach the asymptotic performance. By contrast, the model-driven PG can greatly reduce the reliance on inaccurate Q-value estimate and converge fast. However, its asymptotic performance is inevitably damaged by the model predictive errors. That is because the model-driven PG actually  changes the original objective function \eqref{eq.objective_function} by replacing the environment dynamics with the model. Next, we will instead investigate the unified PG to give insight in terms of its bias.

\subsection{Bias analysis of the data/model PGs}\label{subsection.upper_bound}
In this section, we will derive the error of the unified PG and discuss its relation with model predictive errors and Q-value estimate errors, so as to provide theoretical guidance for the development of the MPG algorithm. For simplicity, we will derive the upper bound of the error in the case of deterministic dynamics and models. 
Besides, throughout this paper, we will use $\Vert \cdot \Vert$ to denote the Euclidean norm $\Vert \cdot \Vert_2$.

The bias of the unified PG is as follows,
\begin{equation}
\label{eq.unified_PG_bias}
\begin{aligned}
    &\Vert \nabla_{\theta}J_n(\theta) - \nabla_{\theta}J(\theta) \Vert\\
    =&\bigg\Vert \Edl{s_t}{\nabla_{\theta}\pi_{\theta}\bigg(\sum_{l=t}^{t+n-1} \gamma^{l-t} (\nabla_{a_t}\hat{r}_l-\nabla_{a_t} r_l)+}\\ &\qquad\qquad\qquad\rVert{\rc{\gamma^{n}(\nabla_{a_t} \hat{Q}_{\omega, t+n}-\nabla_{a_t} q^{\pi_{\theta}}_{t+n})\bigg)}}.
\end{aligned}
\end{equation}
where we expand $q^{\pi_{\theta}}(s, a)$ in \eqref{eq.pg} by $n$-step Bellman recursion, and denote $q^{\pi_{\theta}}_{t+n}:=q^{\pi_{\theta}}(s_{t+n},a_{t+n}), n>0$. All gradient terms in \eqref{eq.unified_PG_bias} can be easily calculated according to the chain rule. See Appendix \ref{appendix.chain_rule} for details.

To further derive the bias of the unified PG, we first define the model predictive error and the value estimate error. First, the model error is defined as the state predictive error, i.e., $\delta(s, a) = p(s, a)-f(s, a)$. About the long-term property of this error, we have the following condition:
\begin{condition}\label{condition.model_error_condition}
The $i$-step model predictive error grows linearly with prediction horizon $i$, i.e., $\sup_{s_t\in \mathcal{S}, a_t\in \mathcal{A}}\Vert s_{t+i}-\hat{s}_{t+i}\Vert \le c_m i$, where $c_m$ is a constant.
\end{condition}
In following analysis, the growing rate of the model predictive error, i.e., $c_m$, will be regarded as an important representation of the model error.

Second, the Q-value estimate error refers to the distance between the Q network $Q_{\omega}$ and Q-value function $q^{\pi_{\theta}}$. The estimate error originates from two parts. First of all, the Q network is learned from interactive data that contains state noise $\epsilon$. The noise essentially changes the environment dynamics, making the true Q-value function inaccessible. We denote the noisy Q-value function under the noisy dynamics as $q^{\pi_{\theta},\epsilon}$. Then, the Q-value function need to be fitted by a parameterized function, where the approximate error is unavoidable. To quantize this error, we give the gradient upper bound as the indicator, as shown in Condition \ref{condition.data_error_condition}:

\begin{condition}\label{condition.data_error_condition}
The difference of the gradient of the approximate function and the noisy Q-value is bounded.
\begin{equation}
\nonumber
\begin{aligned}
\sup_{s \in \mathcal{S}}\Vert \nabla_s Q_{\omega}(s,\pi_{\theta}(s)) - \nabla_s q^{\pi_{\theta}, \epsilon}(s, \pi_{\theta}(s)) \Vert &= B_{qs}\\
\sup_{s \in \mathcal{S}, a \in \mathcal{A}}\Vert \nabla_a Q_{\omega}(s,a) - \nabla_a q^{\pi_{\theta}, \epsilon}(s, a) \Vert &= B_{qa}
\end{aligned}
\end{equation}
where the bounds are expected to be small when we have well estimated Q-values, and vice versa.
\end{condition}

The above condition describes the relation between the approximate function and the noisy Q-value function. To bridge the approximate function and the true Q-value function, we further explore the relation between the noisy and true Q-value function by the following lemma:
\begin{lemma}\label{lemma.data_error}
Given regular conditions of continuity and boundedness shown in Appendix \ref{appendix.regular_conditions}, we have
\begin{equation}
\nonumber
\begin{aligned}
&\sup_{s\in\mathcal{S}}\Vert \nabla_s q^{\pi_{\theta}}(s, \pi_{\theta}(s)) - \nabla_s q^{\pi_{\theta},\epsilon}(s, \pi_{\theta}(s))\Vert \le o(\mathbb{E}_{\epsilon}\left\{\Vert\epsilon\Vert\right\})\\
&\sup_{s\in\mathcal{S}, a\in\mathcal{A}}\Vert \nabla_a q^{\pi_{\theta}}(s, a) - \nabla_a q^{\pi_{\theta},\epsilon}(s, a)\Vert \le o(\mathbb{E}_{\epsilon}\left\{\Vert\epsilon\Vert\right\})
\end{aligned}
\end{equation}
\end{lemma}
\begin{proof}
See Appendix \ref{appendix.proof_of_lemma_data_error}.
\end{proof}

The lemma shows the gradient difference of the Q-value caused by the observation noise can be bounded, which together with the condition \ref{condition.data_error_condition} helps to establish gradient difference between $Q_\omega$ and $q^{\pi_{\theta}}$ when deriving the following lemma. So far, we have presented necessary elements about the model and data error. Then, using it, we study the bias of the unified PG \eqref{eq.unified_PG_bias}. Formally, the model predictive error involves in all future states, thus influences both the term $\nabla_{a_t}\hat{r}_l-\nabla_{a_t}r_l$ and $\nabla_{a_t}\hat{Q}_{\omega,t+n}-\nabla_{a_t}q^{\pi_{\theta}}_{t+n}$. On the other hand, the value estimate error only influences the latter term. Mathematically, we derive the upper bound of these two terms regarded with the quantitative model and data errors:
\begin{lemma}\label{lemma.r_q_grad_upper_bound}
Given Condition \ref{condition.model_error_condition}, Condition \ref{condition.data_error_condition} and regular conditions of continuity and boundedness shown in Appendix \ref{appendix.regular_conditions}, we have
\begin{equation}
\nonumber
\begin{aligned}
&\sup_{s_t\in \mathcal{S}, a_t\in \mathcal{A}}\Vert \nabla_{a_t}\hat{r}_{t+k} - \nabla_{a_t}r_{t+k}\Vert \le \left\{
    \begin{array}{lc}
        o((c_m k)^k), & k \ge 1\\
        0, & k=0\\
    \end{array}\right. \\
&\sup_{s_t\in \mathcal{S}, a_t\in \mathcal{A}}\Vert \nabla_{a_t} \hat{Q}_{\omega, t+k}-\nabla_{a_t} q^{\pi_{\theta}}_{t+k}\Vert \\
&\qquad\qquad\qquad\le \left\{ 
    \begin{array}{lc}
        o((c_m k)^k + B_{qs} +\mathbb{E}_{\epsilon}\left\{\Vert\epsilon\Vert\right\}), & k \ge 1\\
        o(B_{qa}+\mathbb{E}_{\epsilon}\left\{\Vert\epsilon\Vert\right\}), & k=0\\
    \end{array}\right. \\
\end{aligned}
\end{equation}
\end{lemma}

\begin{proof}
See Appendix \ref{appendix.proof_of_lemma_r_q_grad_bound}.
\end{proof}

Lemma \ref{lemma.r_q_grad_upper_bound} shows that the model predictive error and the value estimate error determine the upper bound of the reward gradient and the Q-value gradient difference. Also, it reveal that the model inaccuracy $c_m$ can be exponentially amplified by the predictive horizon $k$. Note that Lemma \ref{lemma.r_q_grad_upper_bound} only analyzes key components of the unified PG bias \eqref{eq.unified_PG_bias}. With the help of this lemma, we further acquire the unified PG bias by the following theorem.

\begin{theorem}\label{theorem.bias_upper_bound}
Given Condition \ref{condition.model_error_condition}, Condition \ref{condition.data_error_condition} , and regular conditions of continuity and boundedness shown in Appendix \ref{appendix.regular_conditions}, we have
\begin{equation}
\nonumber
\begin{aligned}
&\Vert \nabla_{\theta}J_n(\theta) - \nabla_{\theta}J(\theta) \Vert \le\\
&\left\{ 
    \begin{array}{lc}
        o(n(\gamma c_m n)^n + \gamma^n (B_{qs}+\mathbb{E}_{\epsilon}\left\{\Vert\epsilon\Vert\right\})), & n \ge 1\\
        o(B_{qa}+\mathbb{E}_{\epsilon}\left\{\Vert\epsilon\Vert\right\}), & n=0\\
    \end{array}\right. \\
\end{aligned}
\end{equation}
\end{theorem}
\begin{proof}
See Appendix \ref{appendix.proof_of_bias_upper_bound}.
\end{proof}
From the Theorem \ref{theorem.bias_upper_bound}, it can be seen that the upper bound of the unified PG bias is determined by model predictive error $c_m$, Q-value estimate error (including observation error $\epsilon$ and function approximate error $B_{qs},B_{qa}$), and the predictive horizon $n$. Specifically, when the prediction horizon $n=0$, i.e., the data-driven PG \eqref{eq.pg_data}, this bias is only influenced by the Q-value estimate eror, i.e., $B_{qa}+\mathbb{E}_{\epsilon}\left\{\Vert\epsilon\Vert\right\}$. When the prediction horizon $n\ge1$, the unified PG bias is related to both the model error $n(\gamma c_m n)^n$ and the Q-value estimate error $\gamma^n (B_{qs}+\mathbb{E}_{\epsilon}\left\{\Vert\epsilon\Vert\right\})$. And, with the increase of the prediction horizon, the PG bias caused by the model predictive error grows exponentially. On the contrary, that caused by the Q-value estimate error is controlled by the discount factor $\gamma$ and decreases exponentially. When the prediction horizon increases to infinity, the PG error is only related to model errors. However, in that case, the model errors can make the unified PG arbitrarily bad.

At a high level, the Theorem \ref{theorem.bias_upper_bound} gives an explicit expression between the PG bias and the predefined data and model errors. By the conclusion, we can gain some insight about the relative values of the data-driven PG bias and the model-driven PG bias: At the beginning stage of training, the Q-value function is initialized randomly, offering meaningless gradients. In this stage, the function approximate error ($B_{qs},B_{qa}$) dominates, so we can safely omit the model predictive error term $n(\gamma c_m n)^n$ and get that the model-driven PG bias is approximately $\gamma^n$ times that of the data-driven PG. At the late stage of training, the Q-value function gradually converges, making the Q-value estimate error negligible. In this stage, instead, the model predictive error dominates. We can discard the data-related term and find that the model-driven PG bias is exponentially greater than that of the data-driven PG. These findings can be applied to the weight design of the MPG.

\subsection{Mixed policy gradient}
Building upon the conclusions derived in the preceding two sections, we further deduce the mathematical form of MPG. Utilizing Equation \eqref{eq.MPG_form} and Theorem \ref{theorem.model_pg}, we obtain:
\begin{equation}
\label{eq.MPG_objective}
\begin{aligned}
    \nabla_{\theta}J^{\text{Mixed}}(\theta) &= w_{\text{data}} \nabla_{\theta}J^{\text{Data}}(\theta) +w_{\text{model}}\nabla_{\theta}J^{\text{Model}}(\theta) \\
    &=w_{\text{data}} \nabla_{\theta}J_0(\theta) + w_{\text{model}} \nabla_{\theta}J_{\infty}(\theta)
\end{aligned}
\end{equation}
where $w_{\text{model}}$ fuses the constant value of $1/(1-\gamma)$ in the second equation. Then we can design proper weights according to Theorem \ref{theorem.bias_upper_bound}. As presented in the next section, $w_{\text{data}}$ and $w_{\text{model}}$ are designed as a function of policy update times $i$, so as to align with the dynamic relations between $\nabla_{\theta}J_0$ and $\nabla_{\theta}J_\infty$.
However, constructing an infinite Bellman recursion is impractical, so we replace $\nabla_{\theta}J_{\infty}(\theta)$ with $\nabla_{\theta}J_{H}(\theta)$, where $H$ is a positive integer.

To calculate the MPG, we first derive its corresponding loss function:
\begin{equation}
\label{eq.MPG_objective}
\begin{aligned}
J^{\text{Mixed}}(\theta)
=& w_{\text{data}} J_0(\theta) + w_{\text{model}} J_{\infty}(\theta)\\
=&\Ed{s_t}{w_{\text{data}}\hat{Q}_{w, t} + w_{\text{model}} \sum_{l=t}^{\infty} \gamma^{l-t}\hat{r}_l}
\end{aligned}
\end{equation}
where the state $s_t$ obeys $\rho^{\pi_{\theta}}$ by default, which requires to collect data using policy $\pi_{\theta}$. According to \cite{barto2020looking}, the state distribution can be safely replaced with that of an arbitrary policy. That naturally transforms the MPG into a more general off-policy version, allowing us to use states generated by history policies.
The computing process of the MPG is as follows: first, sample a batch of state $s_t$ from a buffer $\mathcal{B}$, where all history experience is stored. Then, compute data and model loss for each state and decide the weights $w_{\text{data}}$ and $w_{\text{model}}$. Finally, compute the loss function \eqref{eq.MPG_objective} and take its gradient to get the MPG. The complete algorithm is shown in Algorithm \ref{alg.mpg}.

The MPG algorithm can be interpreted as a way to balance the variance of the data-driven loss and the bias of the model-driven loss, as shown in Figure \ref{fig.interpretation}. The data-driven PG methods construct $J_0(\theta)$ relying on the estimated Q-value. Because the Q-value changes with policies during the learning process, $J_0(\theta)$ is served as the local approximation of the true loss \eqref{eq.objective_function}, which is inaccurate with a large variance in the early training phase, resulting in bad policy updates. But with the Q-value converges, the local approximation is of high quality, then a local optimal solution can be found under certain conditions\cite{konda1999actor,haarnoja2018soft}. For model-driven PG, the loss function $J_{\infty}(\theta)$ is only determined by the model which is invariant during training, but it is biased due to the mismatch between the model and the true environment dynamics, leading to a bad solution. The MPG algorithm mixes the two loss functions dynamically by weighted average to construct a better approximation of the true loss.

\begin{figure*}[htbp]
\centerline{\includegraphics[width=0.55\linewidth]{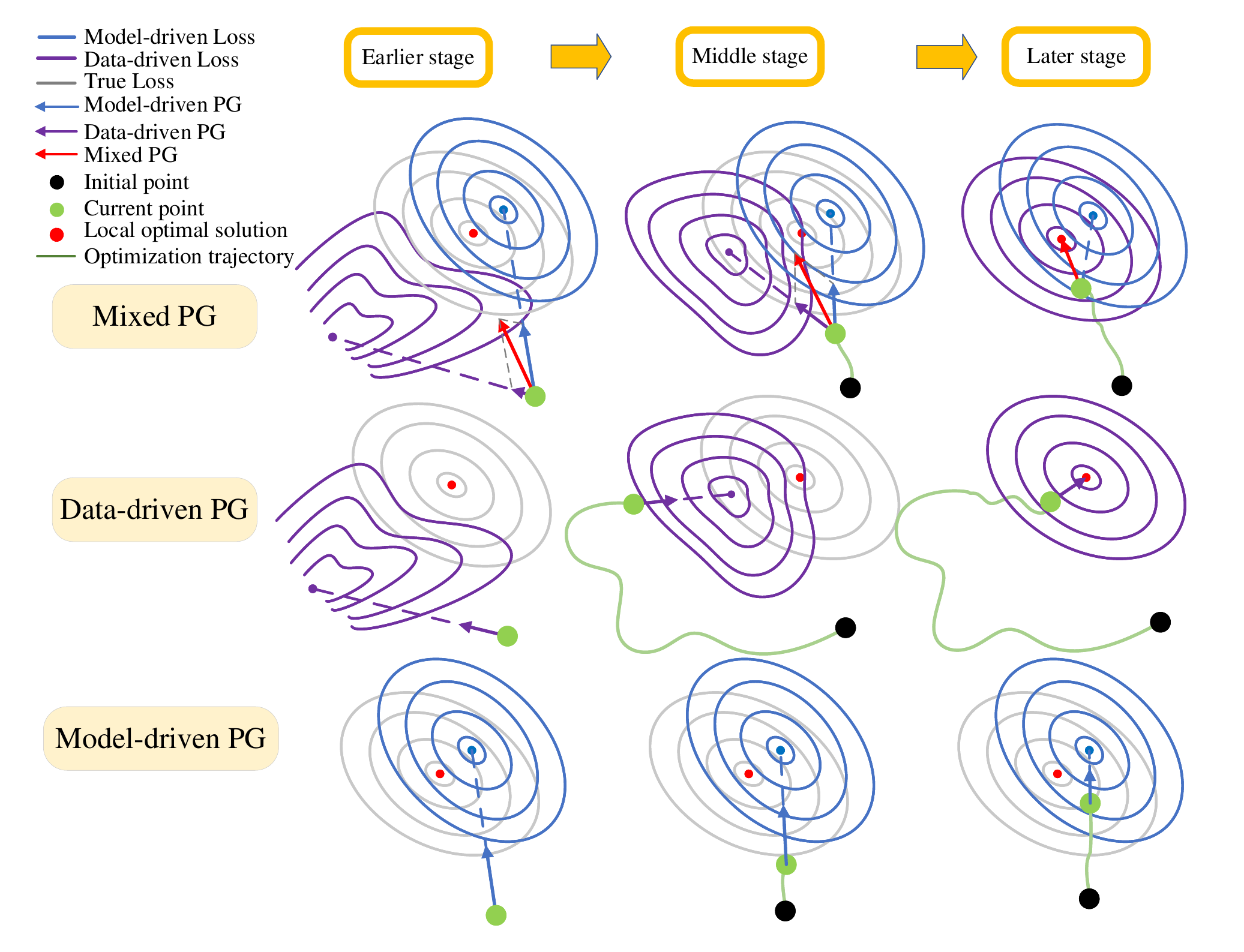}}
\caption{Interpretation of MPG. The loss of data-driven PG converges to the true loss along training process but has large variance in the earlier state. The loss of model-driven PG has low variance but is biased. Mixed PG dynamically adjusts the weights of the data loss and the model loss to construct a better approximation of the true loss function.}
\label{fig.interpretation}
\end{figure*}

\section{Practical algorithm}
\subsection{Weighting method}

To develop the practical algorithm, we first design the PG weights according to the Theorem \ref{theorem.bias_upper_bound}. The basic idea is to realize the worst case optimization, i.e., calculating weights by their respective bias upper bounds. However, some needed quantities to compute the exact upper bounds are hardly obtained, such as the growing rate of the model predictive error $c_m$, the function approximate bound $B_{qs},B_{qa}$, etc. Therefore, we instead focus on the exponential relative relation between the data-driven and the model-driven PGs in different stages of training. Specifically, we propose a rule-based upper ($e_\text{data}$ and $e_\text{model}$) bound for the two types of PG ($\nabla_{\theta}J_0$ and $\nabla_{\theta}J_\infty$) to mimic the same relative relation. The mathematical form of the rule-based error is:
\begin{equation}\label{eq.rule-based_errors}
\begin{aligned}
(e_\text{data}(i),e_\text{model}(i)) &=\left\{
\begin{array}{ll}
 (1,\lambda^H(i)), & i\le T/2 \\        (\lambda^H(i),1), &i>T/2\\
\end{array} \right. ,\\
\lambda(i)
&=\left\{
\begin{array}{ll}
 (1-\eta)+\frac{2i\eta}{T}, & i\le T/2 \\       
 1-\frac{2(i-T/2)\eta}{T}, &i>T/2\\
\end{array}
\right.
\end{aligned}
\end{equation}
where $i$ is the policy iteration times, $T$ is the max iterations, $H$ is a positive integer to establish exponential effect between data and model errors, $\lambda$ is the base number change from $1-\eta$ to 1, and then back to $1-\eta$ with the increase of iteration number, $\eta$ determines the max rate of exponential decay. By such design, the weights can change smoothly with the increase of policy iteration number. Based on the rule-based bias, the weights is finally got by softmax normalization:
\begin{equation}\label{eq.rule-based_weights}
w_\dag(i) = \mathrm{Softmax}(1/e_\dag(i)), \dag\in\{\text{data}, \text{model}\}
\end{equation}

\begin{algorithm}[htbp]
  \caption{Mixed Policy Gradient}
  \label{alg.mpg}
\begin{algorithmic}
  \STATE {\bfseries Initialize:} critic network $Q_{\omega}$ and actor network $\pi_{\theta}$ with random parameters $\omega$, $\theta$, target parameters $\omega' \leftarrow \omega, \theta' \leftarrow \theta$, batch size $N$, initial buffer $\mathcal{B}$ with $N$ samples, update interval $m$, exploration noise $\epsilon$, target smoothing coefficient $\tau$, learning rates $\beta_{\omega}, \beta_{\theta}$
  \FOR{each iteration $k$}
      \STATE Collect a batch of transitions using $\pi_{\theta}$, store in $\mathcal{B}$
      \STATE Sample $N$ transitions $\{s_i, a_i, r_i, s_{i+1}\}_{i=0:N-1}$ from $\mathcal{B}$
      \STATE Update value $\omega \leftarrow \omega - \beta_\omega \nabla_{\omega}J_Q(\omega)$ using \eqref{eq.value_grad}
      \IF{$k \mod m$}
      \STATE Compute $w_{\text{data}}(k), w_{\text{model}}(k)$ by \eqref{eq.rule-based_weights}
      \STATE Compute $J^{\text{Mixed}}(\theta)$ using \eqref{eq.MPG_objective}
      \STATE Update policy $\theta \leftarrow \theta + \beta_{\theta} \nabla_{\theta}J^{\text{Mixed}}(\theta)$
      \STATE Update target networks:
      \STATE \quad $\omega' \leftarrow \tau \omega + (1-\tau) \omega'$
      \STATE \quad \ $\theta' \leftarrow \tau \theta + (1-\tau) \theta'$
      \ENDIF
  \ENDFOR
\end{algorithmic}
\end{algorithm}

\subsection{Value learning}
To learn the Q-value function, we minimize the mean squared error between the approximate value function and target values, i.e.,
\begin{equation}
\nonumber
\min_{\omega} J_Q(\omega) = \Ed{s,a}{\frac{1}{2}(R^{\text{Target}}(s, a) - Q_{\omega}(s,a))^2},
\end{equation}
where $s,a$ obey arbitrary joint distributions. Depending on the way to calculate $R^{\text{Target}}$, we develop two variants of the MPG, i.e., MPG-v1 and MPG-v2, where the former uses $n$-step TD method, and the latter employs clipped double Q method developed by Fujimoto \emph{et al.} \cite{fujimoto2018addressing}.


\begin{equation}\label{eq.value_target}
\begin{aligned}
&R^{\text{Target}}(s_t, a_t) =\\
&\left\{
\begin{array}{ll}
         \sum_{l=t}^{t+n-1} \gamma^{l-t} r_l + \gamma^{n} Q_{\omega'}(s_{t+n}, \pi_{\theta'}(s_{t+n})), & \text{MPG-v1} \\
         r_t + \gamma \min\limits_{i=1,2}Q_{\omega'_i}(s_{t+1}, \pi_{\theta'}(s_{t+1})), & \text{MPG-v2}\\
\end{array} \right.\\
\end{aligned}
\end{equation}
where $\omega'$ and $\theta'$ are target Q-value function and target policy to stabilize learning. For MPG-v1, the calculation of the $n$-step TD requires running the current policy in the environment starting from each state in the batch. The procedure can be one of the bottlenecks of update throughput if the environment has low sampling efficiency. To solve the problem, we propose a batch reuse technique to improve sample utilization, in which a batch of data and its correlated target are used several times to compute gradients. Besides, the problem can also be alleviated by updating asynchronously.
The value parameters can be optimized with gradients
\begin{equation}\label{eq.value_grad}
\nabla_\omega J_Q(\omega) = -\Ed{s,a}{(R^{\text{Target}}(s, a) - Q_\omega(s,a))\nabla_\omega Q_\omega(s,a)}.
\end{equation}

\subsection{Asynchronous learning architecture}\label{subsection.architecture}
The wall-clock time to converge equals to the product of the required iteration number and the time consumption of per iteration. The former is determined completely by the algorithms, while the latter (inversely proportional to the update throughput) consists of two parts in a serial learning architecture. One is the time for computing the PG, and the other is the time for applying it. Although MPG has great convergence speed w.r.t. iteration number, the computing time of MPG is roughly proportional to the predictive horizon $H$, resulting in lower computation efficiency than data-driven PG methods.
We propose an asynchronous learning architecture to reduce the time consumption per iteration by eliminating the time for waiting a PG to be computed, so as to match the update throughput of data-driven PG methods.
In the architecture, Buffers, Actors and Learners are all distributed across multiple processes to improve the efficiency of replay, sampling, and PG computation, as shown in Fig. \ref{fig.architecture}.

The Optimizer is the core module, which is in charge of creating and coordinating other processes. The Optimizer owns the latest parameters, a gradient queue, an experience queue and an update thread, which keeps fetching gradients from the gradient queue and updating the parameters. The Optimizer also creates an Evaluator and several Actors, Buffers and Learners in separate processes. Each Actor owns local parameters and the environment to generate experience. They asynchronously synchronize the parameters from the Optimizer, and send the generated experience to a random Buffer. The Buffers receive the experience collected by the Actors and keep sampling batch data to the experience queue. Each Learner owns local parameters and a loss function defined by an RL algorithm to compute the gradients of the parameters. They asynchronously synchronize the parameters from the Optimizer, sample experience from the experience queue, and send gradients to the gradient queue for updating. The Evaluator contains local parameters, an environment and metrics. It asynchronously synchronizes the parameters from the Optimizer, runs the trained policy in the environment to compute the metrics for evaluating the policies along the training process.



\begin{figure}[htbp]
\centerline{\includegraphics[width=0.95\linewidth]{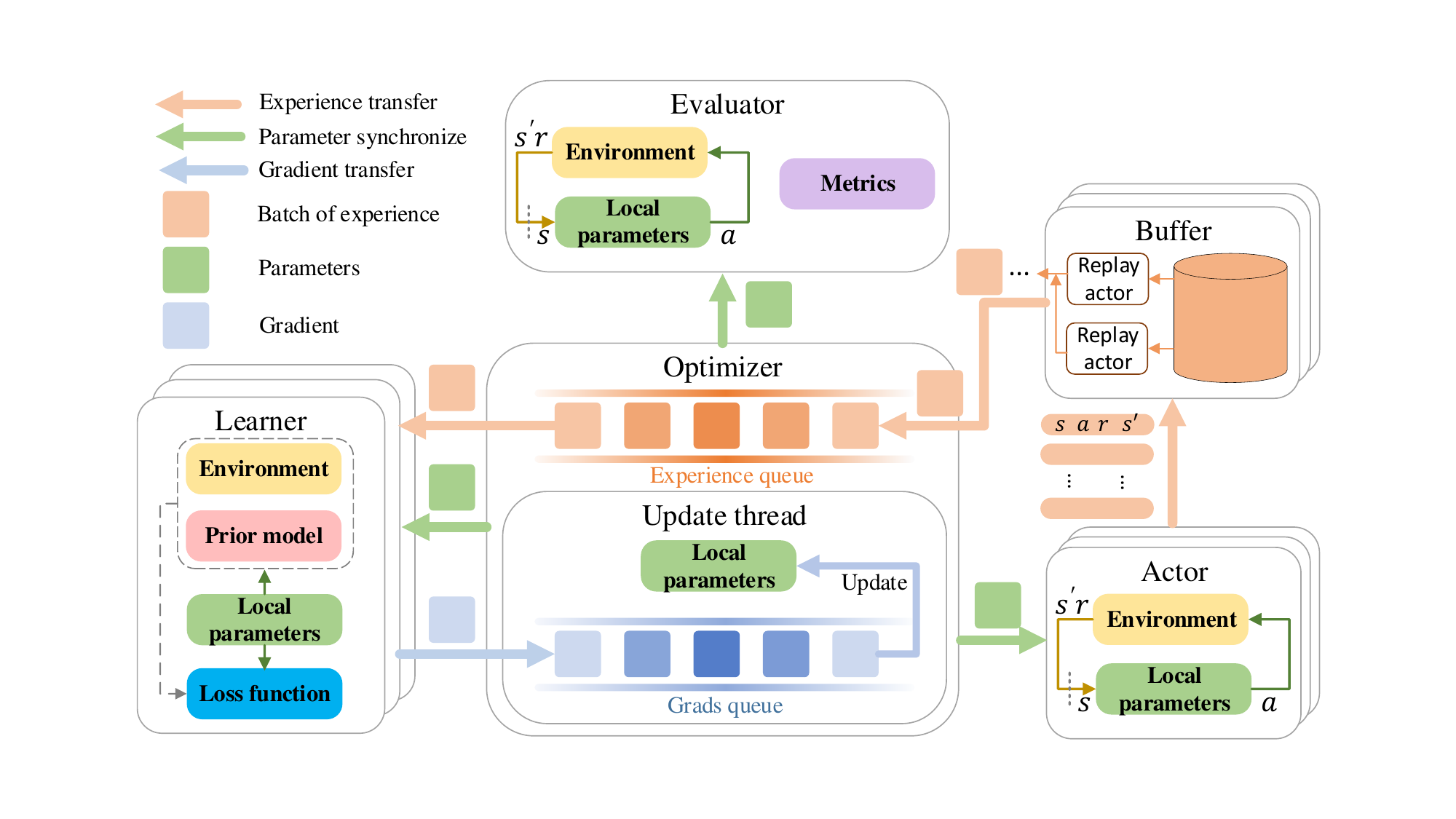}}
\caption{Asynchronous learning architecture. Buffers, Actors and Learners are all distributed across multiple processes to improve the efficiency of replay, sampling, and PG computation. The time consumption per iteration of MPG can be significantly reduced by employing more Learners.}
\label{fig.architecture}
\end{figure}

\section{Simulations}
\subsection{Simulators and prior models}
We conduct experiments on both path tracking task and inverted pendulum task, as shown in Fig. \ref{fig.tasks}. In the path tracking task, the autonomous vehicle is required to track the given path and velocity to minimize the tracking errors. We develop a highly parallelized simulator for this task, which allows hundreds of agents to run simultaneously. The inverted pendulum task is adapted from the Gym package to make its reward an explicit function of states and actions, which is also used by its prior model. The task uses MuJoCo as its simulator. 
The prior models need to be deterministic and differentiable w.r.t the state and action to construct the MPG through BPTT. To test the effectiveness of our methods, we take several measures to enlarge the mismatch between the model and the true dynamics in the simulator, such as adding biased noise in the transition model, using different discrete time step, etc.

\begin{figure}[h]
\centering
\captionsetup[subfigure]{justification=centering}
\subfloat[]{\label{fig.path_tracking_task}\includegraphics[width=0.19\textwidth]{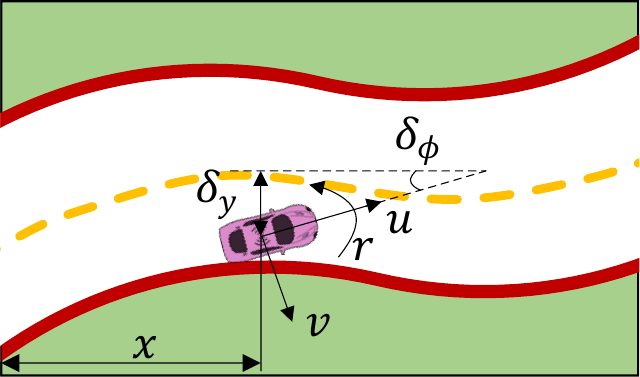}}
\subfloat[]{\label{fig.inverted_pendulum_task}\includegraphics[width=0.2\textwidth]{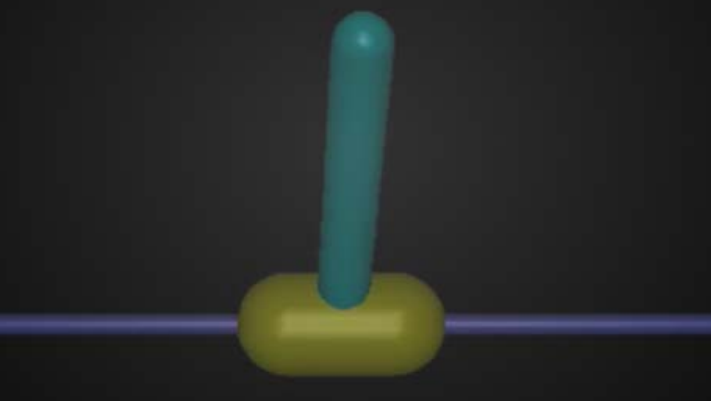}}
\caption{Tasks. (a) Path tracking task: $(\mathcal{S},\mathcal{A})\subset\mathbb{R}^6\times\mathbb{R}^2$. (b) Inverted pendulum task: $(\mathcal{S},\mathcal{A})\subset\mathbb{R}^4\times\mathbb{R}^1$.}
\label{fig.tasks}
\end{figure}

\subsection{Experiments and results}
We compare our algorithm against several state-of-the-art data-driven algorithms, namely deep deterministic policy gradient (DDPG), twin delayed deep deterministic policy gradient algorithm (TD3), soft actor-critic (SAC), and model-driven algorithm, adaptive dynamic programming (ADP). All the methods belong to off-policy category. For fair comparison, we upgrade the value learning of DDPG and ADP as MPG (v1) to get $n$-step DPG and $n$-step ADP, while keeping their respective policy learning methods.

\subsubsection{Performance}
All the algorithms are implemented in the asynchronous learning architecture proposed in section \ref{subsection.architecture}, including 2 Actors, 2 Buffers, and 12 Learners. For action-value functions and policies, we use a fully connected neural network (NN) with 2 hidden layers, consisting of 256 units per layer, with Exponential Linear Units (ELU) each layer \cite{clevert2015fast}. For stochastic policies, we use a Gaussian distribution with mean and covariance given by a NN, where the covariance matrix is diagonal. The Adam method \cite{kingma2014adam} with a polynomial decay learning rate is used to update all the parameters.

We train 5 different runs of each algorithm with different random seeds, with evaluations every 3000 iterations. Each evaluation calculates the average return over 5 episodes without exploration noise, where each episode has a fixed 200 time steps (100 for inverted pendulum task). We illustrate the algorithm performance in terms of its asymptotic performance and convergence speed. The former is measured by the episode return during training, and the latter is measured by the number of iterations needed to reach a certain goal performance. The curves are shown in Fig. \ref{fig.perf_path_tracking} and Fig. \ref{fig.perf_inverted_pendulum}, and results in Table \ref{tab.perf}. The MPG (v1 and v2), $n$-step DPG, and $n$-step ADP use mixed PG, data-driven PG and model-driven PG for policy learning, respectively. 

Results show that, MPG (v1 and v2) outperform all the baseline methods in terms of the asymptotic performance and convergence speed on the path tracking task. Specifically, $n$-step DPG is the same as MPG except for the usage of data-driven PG, but it is much slower than MPG and converges to a worse policy with larger variance. By contrast, $n$-step ADP uses model-driven PG, leading to the worst asymptotic performance because of model errors, despite its fast convergence. MPG also learns considerably faster than TD3 and SAC with better final performance and smaller variance. This is because that these data-driven baselines compute PG totally depending on the action-value function, which is hard to learn well in the early phase of the training process. On the inverted pendulum task, MPG-v2 yields the best result, but MPG-v1 and $n$-step DPG algorithm fail to learn a feasible policy. The reason is that both of them use $n$-step TD learning, which suffers from large variance on that task, i.e., slightly change in the policy may lead to different orders of magnitude of returns.
The quantitative results attained by MPG indicate that the mixture of data-driven PG and model-driven PG helps to obtain both their benefits.

\begin{figure*}[htbp]
\centering
\captionsetup[subfigure]{justification=centering}
\subfloat[Training curves]{\label{fig.training_curves_path_tracking}\includegraphics[width=0.49\textwidth]{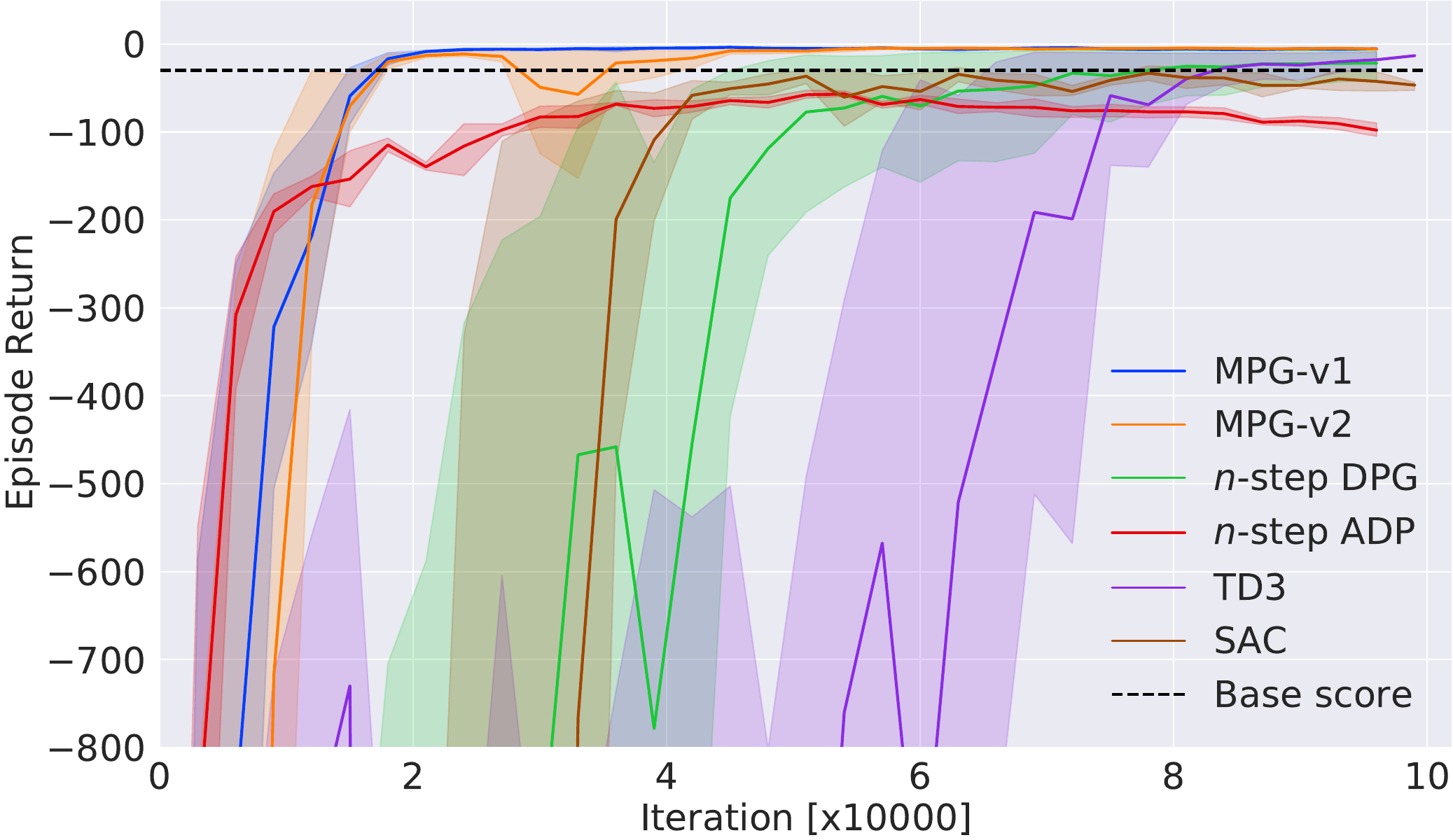}}\quad
\subfloat[Convergence speed]{\label{fig.convergence_speed_path_tracking}\includegraphics[width=0.49\textwidth]{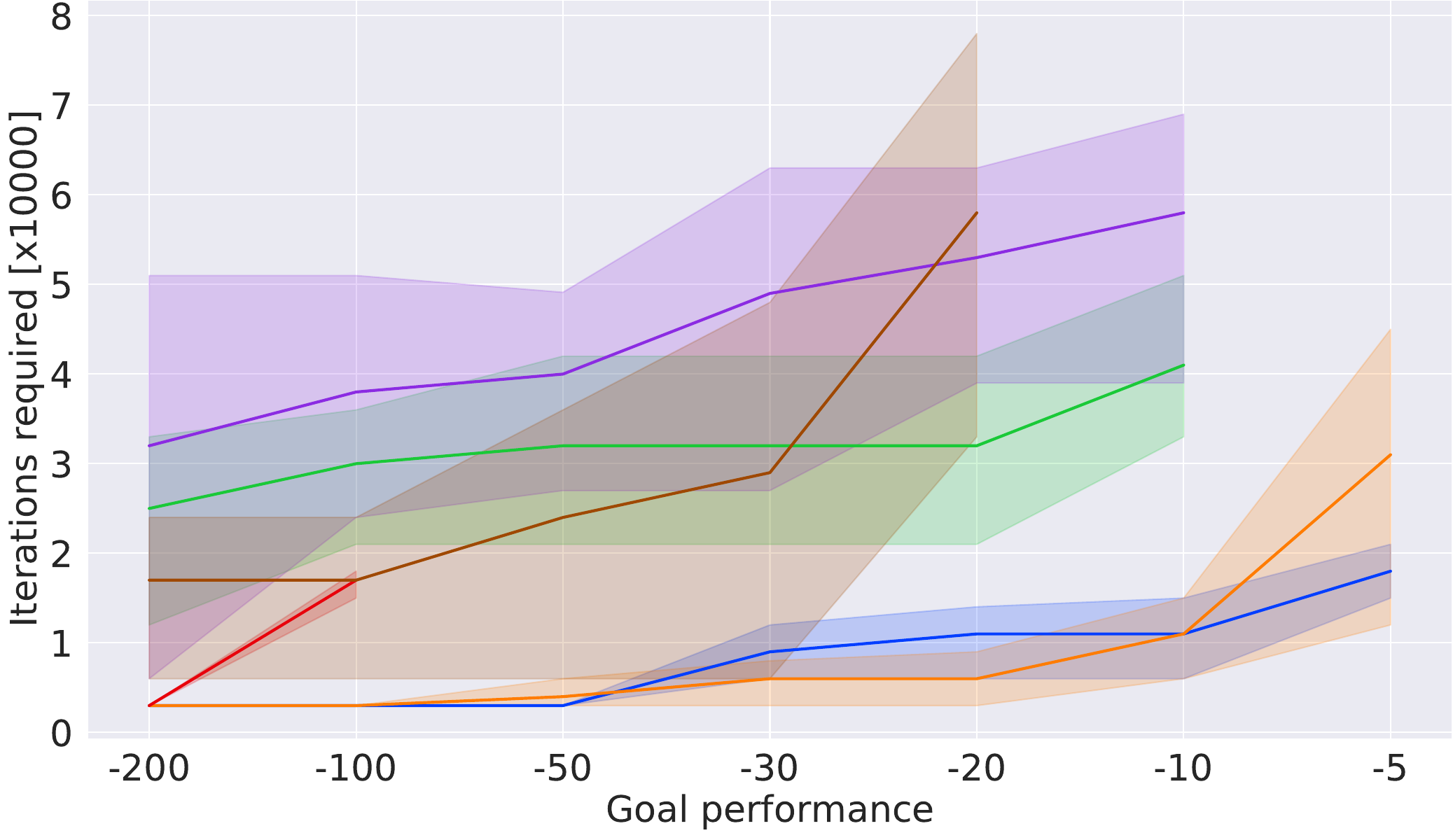}}
\caption{Algorithm comparison in terms of asymptotic performance and convergence speed on the path tracking task. (a) Training curves. The dashed line shows the minimum requirement for the task to work, which is -30 in the task. (b) Convergence speed of different algorithms. The x-coordinate is different goal performance, i.e., episode return, and the y-coordinate is the iteration number needed to reach the goal. The missing part of the curves on some goal performances means that the algorithms never reached these goals during the training process. 
The solid lines correspond to the mean and the shaded regions correspond to 95\% confidence interval over 5 runs.}
\label{fig.perf_path_tracking}
\end{figure*}

\begin{figure*}[htbp]
\centering
\captionsetup[subfigure]{justification=centering}
\subfloat[Training curves]{\label{fig.training_curves_inverted_pendulum}\includegraphics[width=0.49\textwidth]{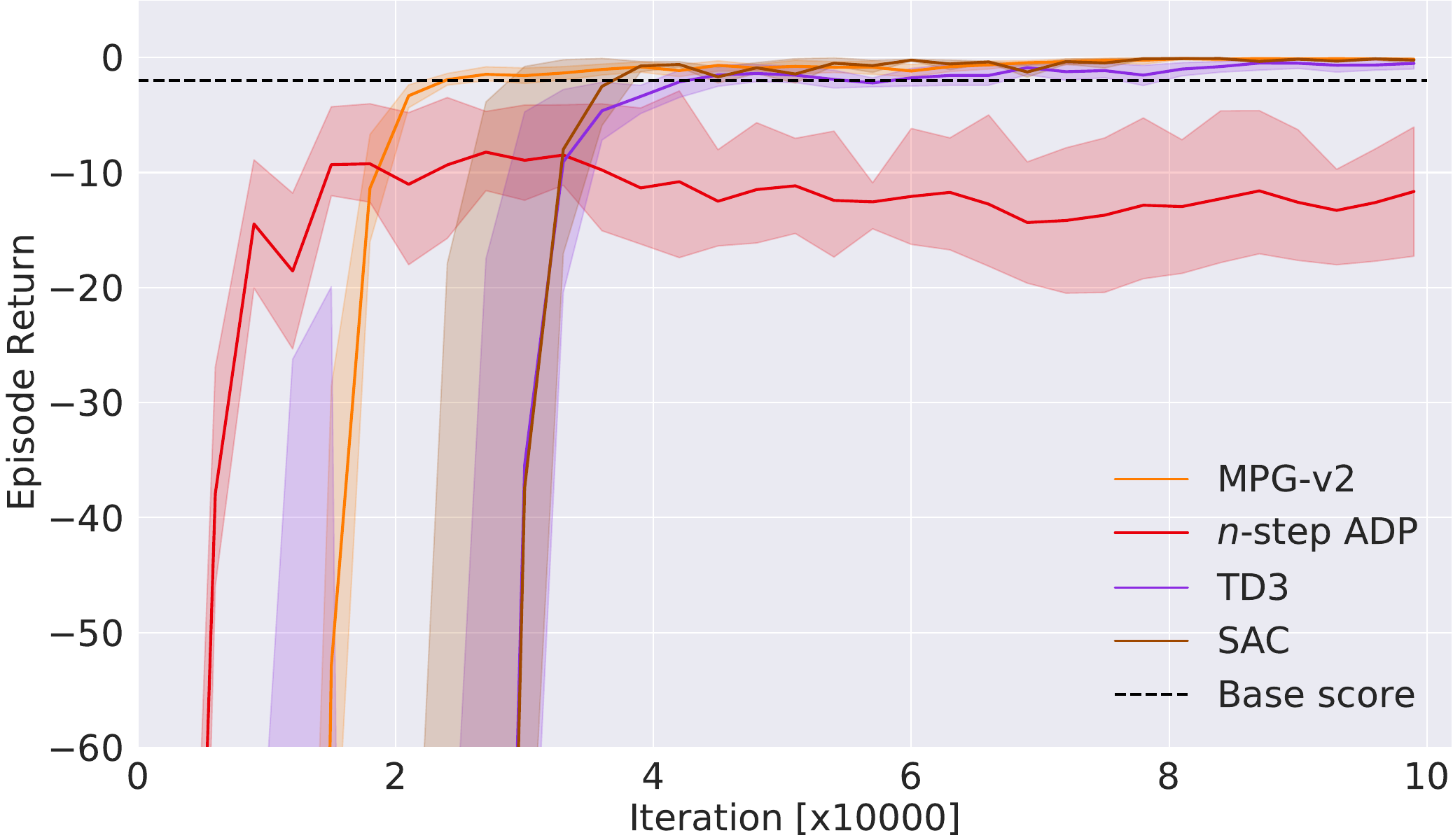}}\quad
\subfloat[Convergence speed]{\label{fig.convergence_speed_inverted_pendulum}\includegraphics[width=0.49\textwidth]{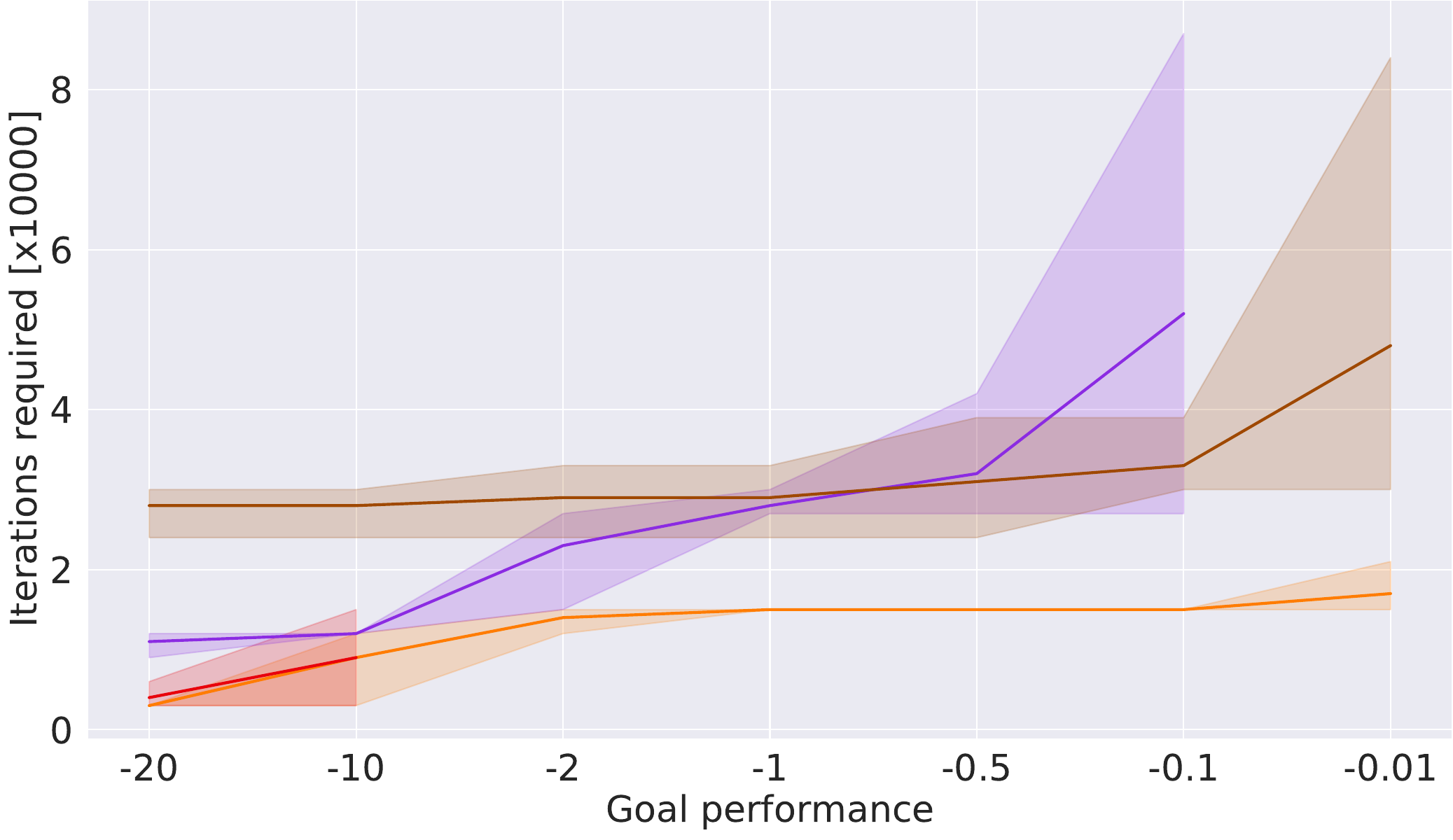}}
\caption{Algorithm comparison in terms of asymptotic performance and convergence speed on the inverted pendulum task. (a) Training curves. The dashed line shows the minimum requirement for the task to work, which is -2 in the task. (b) Convergence speed of different algorithms. The MPG-v1 and $n$-step DPG fail on this task because of the $n$-step TD value learning, so they are not plotted. The solid lines correspond to the mean and the shaded regions correspond to 95\% confidence interval over 5 runs.}
\label{fig.perf_inverted_pendulum}
\end{figure*}

\begin{table*}[h]
\centering
\caption{Asymptotic performance and convergence speed over 5 runs of 100 thousand iterations. The best value for each task is bolded. $\pm$ corresponds to 95\% confidence interval over runs. The convergence speed is measured by the average iterations needed to reach a certain performance.}
\label{tab.perf}
\begin{threeparttable}[h]
\begin{tabular}{lllp{65pt}p{65pt}p{65pt}p{65pt}}
\toprule
Task & Algorithm    & Performance       & Convergence speed (-100, -20)\tnote{*} & Convergence speed (-30, -2)   & Convergence speed (-10, -0.1)  & Convergence speed (-5, -0.01)\\
\midrule
Path  &MPG-v1 (ours)     & \textbf{-3.73 $\pm$1.16}  & \textbf{3000 $\pm$ 0} & 9000 $\pm$ 4898          & \textbf{11000 $\pm$ 7483}  & \textbf{18000 $\pm$ 4898} \\
tracking  &MPG-v2 (ours) & -4.60 $\pm$1.40           &  3000 $\pm$ 0         & \textbf{6000 $\pm$ 4898} & 11000 $\pm$ 7483           & 31000 $\pm$ 27856  \\
          &$n$-step DPG  &-21.89 $\pm$47.87          & 30000 $\pm$ 12961     & 32000 $\pm$ 17204        & 41000 $\pm$ 14966          & -                          \\
          &$n$-step ADP  &-56.83 $\pm$7.44          & 17000 $\pm$ 2828      & -                        & -                          & -                        \\
          &TD3           & -13.30 $\pm$26.50         & 38000 $\pm$ 22090     & 49000 $\pm$ 31496        & 58000 $\pm$ 26981          & -                          \\
          &SAC           & -33.18 $\pm$19.97         & 17000 $\pm$ 15748     & 29000 $\pm$ 34756        & -                          & -                         \\
\midrule
Inverted          &MPG-v2 (ours)  & \textbf{-0.08 $\pm$0.16}  & \textbf{3000 $\pm$ 0} & \textbf{14000 $\pm$ 2828}  & \textbf{15000 $\pm$ 0}  & \textbf{17000 $\pm$ 5656}  \\
Pendulum$\dagger$ &$n$-step ADP   &-8.22 $\pm$7.64          & 4000 $\pm$ 2828       & -                          & -                       & -                        \\
                  &TD3            & -0.48 $\pm$1.35           & 11000 $\pm$ 22090     & 23000 $\pm$ 11313          & 52000 $\pm$ 50990       & -                          \\
                  &SAC            & -0.09 $\pm$0.12           & 28000 $\pm$ 5656      & 29000 $\pm$ 7483           & 33000 $\pm$ 8485        & 48000 $\pm$ 50911       \\

\bottomrule
\end{tabular}
\begin{tablenotes}
     \item[$\dagger$] MPG-v1 and $n$-step DPG fail on this task.
     \item[*] The goal performance for path tracking and inverted pendulum tasks, respectively.
     \item[-] The algorithm never reached the goal performance.
\end{tablenotes}
\end{threeparttable}
\end{table*}

\subsubsection{Time efficiency}
Figure \ref{fig.time_perf} compares the time efficiency of different algorithms on the path tracking task. Results show that the average wall-clock time to compute a gradient of MPG, together with $n$-step ADP, is much higher than $n$-step DPG, TD3 and SAC. This is because the time for BPTT grows linearly with the length of time chain, which is set to be 25 in our experiment. If we employ a synchronous learning architecture, the fast convergence advantage of MPG would be offset by the slow gradient computing time. Therefore, we develop an asynchronous learning architecture to address this problem, in which we use several learners to accelerate the gradient generation. As a result, we reduce the update time of MPG as well as $n$-step ADP to the level of data-driven RL methods, in which there is no need to wait for a gradient to be computed and all the time is used to update the value and policy networks.

\begin{figure*}[htbp]
\centering
\captionsetup[subfigure]{justification=centering}
\subfloat[Computing time of a gradient]{\label{fig.time_grad}\includegraphics[width=0.49\textwidth]{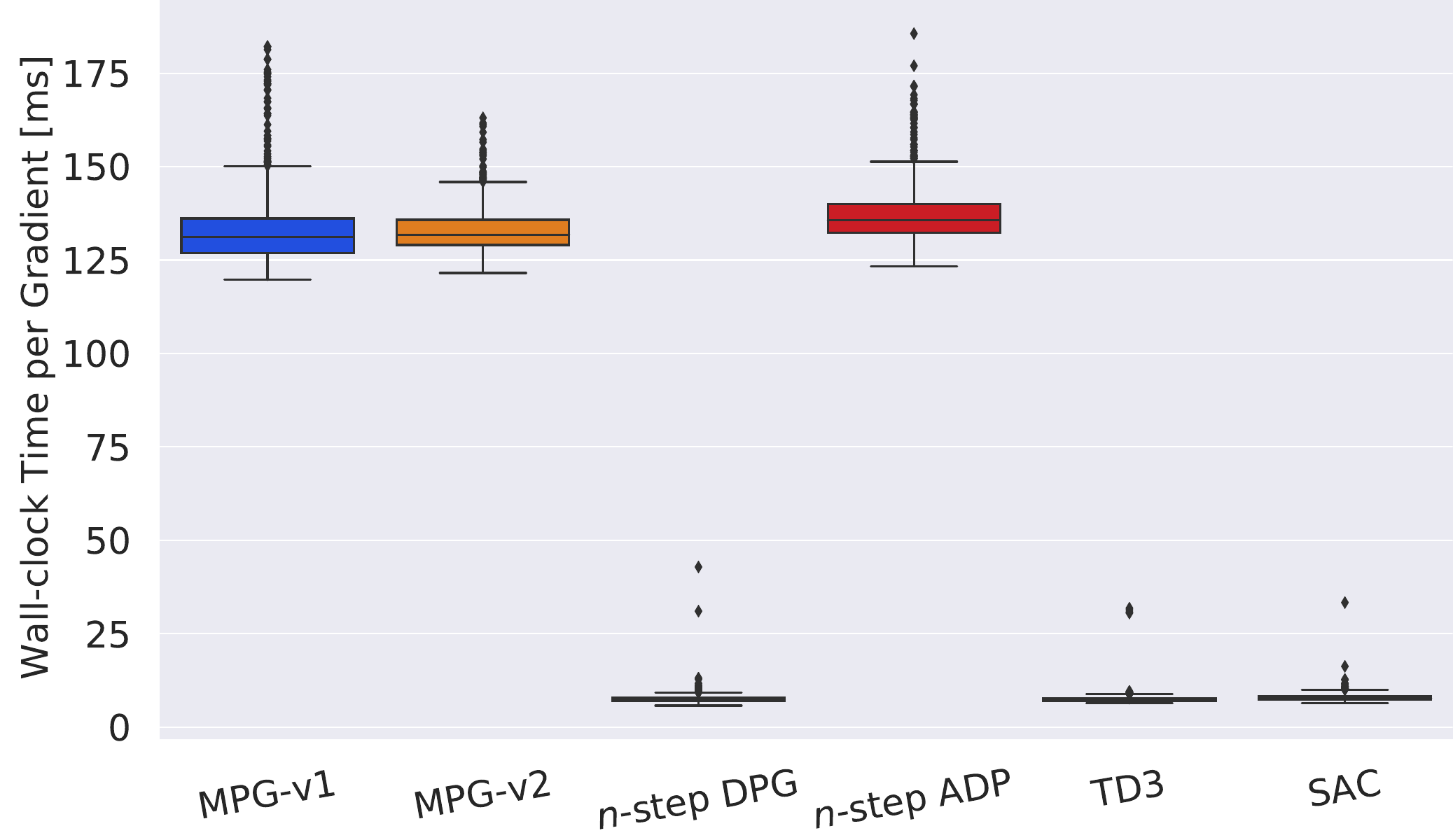}}\quad
\subfloat[Update time for each iteration]{\label{fig.time_update}\includegraphics[width=0.49\textwidth]{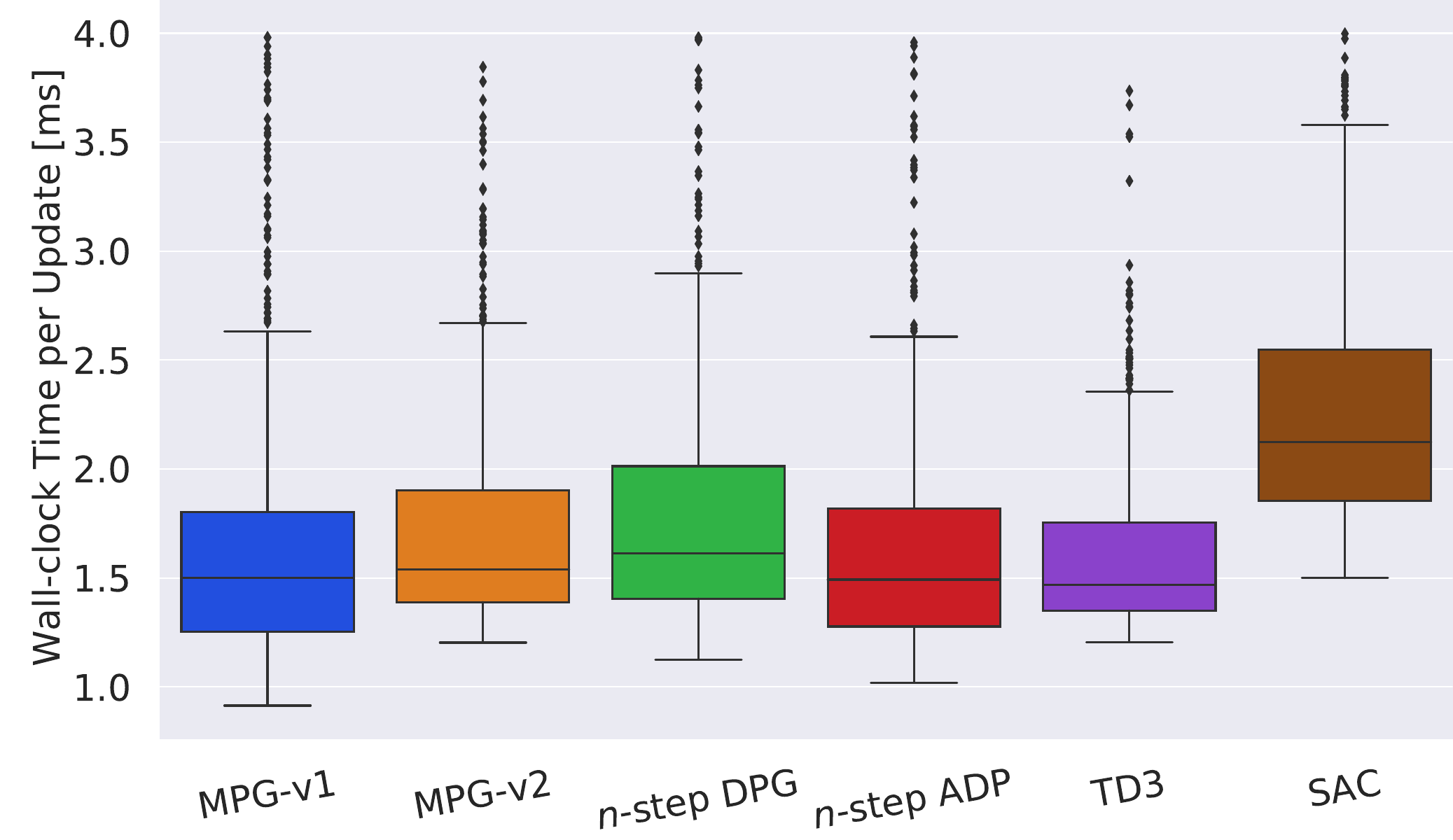}}
\caption{Algorithm comparison in terms of time efficiency on the path tracking task. (a) Computing time of a policy gradient. (b) Update time for each iteration. Each boxplot is drawn based on all evaluations over 5 runs. All evaluations were performed on a single computer with a 2.4 GHz 50 core Inter Xeon CPU.}
\label{fig.time_perf}
\end{figure*}
\section{Conclusion}
In this paper, we presented mixed policy gradient (MPG), an off-policy reinforcement learning algorithm, which uses both empirical data and a prior model so as to get benefits in terms of both asymptotic performance and convergence speed. The MPG is constructed as weighted average of the data-driven and model-driven PGs. The data-driven PG uses the learned action-value function to compute its derivative with respect to the action, whereas the model-driven PG uses the model rollout to computes the gradient by BPTT. Our theoretical results reveal the upper bound of the two PGs' bias, and their exponential relative magnitude during training. Relying on this, we design a rule-based weighting method, which can dynamically balance the bias of the data and model PGs. Besides, an asynchronous learning architecture is build to reduce each update wall-clock time to data-driven level. We evaluate our method on the path tracking task and inverted pendulum task. Results show that the proposed MPG algorithm outperforms all baseline algorithms with a large margin in terms of both asymptotic performance and convergence speed.


\bibliographystyle{IEEEtran}
\bibliography{cite}

\appendices
\section{Proofs}
\subsection{Chain rule}\label{appendix.chain_rule}
\begin{equation*}
\begin{aligned}
    &\nabla_{a_t} \hat{r}_{t+k}\\
    =&\left\{
    \begin{array}{ll}
         \nabla_{a_t} \hat{s}_{t+1} \prod_{j=2}^{k}{\nabla_{\hat{s}_{t+j-1}} \hat{s}_{t+j}} \nabla_{\hat{s}_{t+k}} \hat{r}_{t+k} ,& k\ge2 \\
         \nabla_{a_t}\hat{s}_{t+1} \nabla_{\hat{s}_{t+1}} \hat{r}_{t+1},& k=1\\
         \nabla_{a_t} \hat{r}_t, &k=0\\
    \end{array} \right.\\ 
    &\nabla_{a_t} r_{t+k} \\
    =&\left\{
    \begin{array}{ll}
         \nabla_{a_t} s_{t+1} \prod_{j=2}^{k}{\nabla_{s_{t+j-1}} s_{t+j}} \nabla_{s_{t+k}} r_{t+k} ,& k\ge2 \\
         \nabla_{a_t}s_{t+1}\nabla_{s_{t+1}} r_{t+1},& k=1\\
         \nabla_{a_t} r_t, &k=0\\
    \end{array} \right.\\ 
    &\nabla_{a_t} \hat{Q}_{\omega, t+k} \\
    =&\left\{
    \begin{array}{ll}
         \nabla_{a_t} \hat{s}_{t+1} \prod_{j=2}^{k}{\nabla_{\hat{s}_{t+j-1}} \hat{s}_{t+j}} \nabla_{\hat{s}_{t+k}} \hat{Q}_{\omega, t+k},& k\ge2 \\
         \nabla_{a_t} \hat{s}_{t+1}\nabla_{\hat{s}_{t+1}} \hat{Q}_{\omega,t+1} ,& k=1\\
         \nabla_{a_{t}} \hat{Q}_{\omega,t}, &k=0\\
    \end{array} \right.\\ 
    &\nabla_{a_t} q^{\pi_{\theta}}_{t+k}\\ 
    =&\left\{
    \begin{array}{ll}
         \nabla_{a_t} s_{t+1}  \prod_{j=k}^{2}{\nabla_{s_{t+j-1}} s_{t+j}} \nabla_{s_{t+k}} q^{\pi_{\theta}}_{t+k} ,& k\ge2 \\
         \nabla_{a_t} s_{t+1} \nabla_{s_{t+1}} q^{\pi_{\theta}}_{t+1} ,& k=1\\
         \nabla_{a_t} q^{\pi_{\theta}}_{t}, &k=0\\
    \end{array} \right.\\ 
\end{aligned}
\end{equation*}

\subsection{Proof of the Theorem \ref{theorem.model_pg}}\label{appendix.proof_of_model_pg}
The relation between the data-driven PG and the unified PG is trivial. We will focus on the model-driven part, which requires two assumptions, i.e., $f=p$ and $\rho^{\pi_{\theta}}=d^{\pi_{\theta}}$. With the assumptions, any state that predicted by the model will obey stationary distribution, i.e., $\hat{s}_{t+l}\sim d^{\pi_{\theta}},\forall{l\ge0}$.
For simplicity, we denote
\begin{equation*}
B^j := \Ed{\hat{\tau}(\hat{s}_{t+j-1},\hat{a}_{t+j-1})}{\sum_{l=t+i-1}^{\infty} \gamma^{l-t} \hat{r}_l},
\end{equation*}
then
\begin{align}
    &\nabla_{\theta}J^{\text{Model}}(\theta)\nonumber\\
    =&\Ed{s_t}{\nabla_{\theta} \Ed{\hat{\tau}(s_t)}{\sum_{l=t}^{\infty} \gamma^{l-t} \hat{r}_l}}\nonumber\\
    =& \Ed{s_t}{\nabla_{\theta}\pi_{\theta}\nabla_{a_t} B_1 + \Ed{\hat{s}_{t+1}}{\nabla_{\theta}\pi_{\theta}\nabla_{\hat{a}_{t+1}} B_2} + \dots }\nonumber\\
    =&\sum_{k=0}^{\infty}\Ed{\hat{s}_{t+k}}{\nabla_{\theta}\pi_{\theta}\nabla_{\hat{a}_{t+k}} B_{k+1}}\nonumber\\
    =&\sum_{k=0}^{\infty}\gamma^k \Ed{s_t}{\nabla_{\theta}\pi_{\theta}\nabla_{a_t} B_1}\nonumber\\
    =&\sum_{k=0}^{\infty}\gamma^k \nabla_{\theta}J_{\infty}(\theta)\nonumber\\
    =&\frac{1}{1-\gamma}\nabla_{\theta}J_{\infty}(\theta)\nonumber
\end{align}
where the second equation holds from the chain rule. The third equation holds because each predictive state is independent and obeys stationary distribution. The fourth equation moves the whole trajectory from $\hat{s}_{t+k}$ $k$ steps back with a gain $\gamma^k$, which holds also from the stationary distribution.



\subsection{Regular Conditions}\label{appendix.regular_conditions}
\begin{regular_conditions}\label{conditions.continuity}
(Continuity). \\
$\nabla_s f(s, \pi_{\theta}(s))$ is $L_{\nabla f}$-Lipschitz,\\
$\nabla_{s} r(s, \pi_{\theta}(s))$ is $L_{\nabla r}$-Lipschitz,  \\
$q^{\pi_{\theta}}(s, \pi_{\theta}(s))$ is $L_q$-Lipschitz,\\
$Q_\omega(s, \pi_{\theta}(s))$ is $L_Q$-Lipschitz,\\
$q^{\pi_{\theta}, \epsilon}(s, \pi_{\theta}(s))$ is $L_{q^\epsilon}$-Lipschitz,\\
$\nabla_s q^{\pi_{\theta}}(s, \pi_{\theta}(s))$ is $L_{\nabla q}$-Lipschitz, \\
$\nabla_s Q_\omega(s, \pi_{\theta}(s))$ is $L_{\nabla Q}$-Lipschitz,\\
$\nabla_s q^{\pi_{\theta}, \epsilon}(s, \pi_{\theta}(s))$ is $L_{\nabla q^{\epsilon}}$-Lipschitz.
\end{regular_conditions}

\begin{regular_conditions}\label{conditions.boundedness}
(Boundedness).
\begin{flalign*}
&\sup_{s\in\mathcal{S}, a\in\mathcal{A}}\Vert \delta(s, a) \Vert = B_\delta,\\
&\sup_{s\in \mathcal{S}}\Vert \nabla_s \delta(s, \pi(s)) \Vert = B_{\nabla \delta s},\\
&\sup_{s\in\mathcal{S}, a\in\mathcal{A}}\Vert \nabla_a \delta(s, a) \Vert = B_{\nabla \delta a},\\
&\sup_{s \in \mathcal{S}}\Vert \nabla_s Q_\omega(s,\pi_{\theta}(s)) - \nabla_s q^{\pi_{\theta}, \epsilon}(s, \pi_{\theta}(s)) \Vert = B_{qs},\\
&\sup_{s \in \mathcal{S}, a \in \mathcal{A}}\Vert \nabla_a Q_\omega(s,a) - \nabla_a q^{\pi_{\theta}, \epsilon}(s, a) \Vert = B_{qa},\\
&\sup_{s \in \mathcal{S}}\Vert \nabla_{\theta} \pi_{\theta}(s) \Vert = B_{\nabla \theta},\\
&\sup_{j\ge 1}\sup_{s_t \in \mathcal{S}}\Vert \nabla_{s_t} f(\hat{s}_{t+j}, \pi_{\theta}(\hat{s}_{t+j})) \Vert = B_{\nabla fs},\\
&\sup_{s \in \mathcal{S}, a \in \mathcal{A}}\Vert \nabla_{a} f(s, a) \Vert = B_{\nabla fa},\\
&\sup_{j\ge 1}\sup_{s_t\in\mathcal{S}}\Vert \nabla_{s_t} p(s_{t+j}, \pi_{\theta}(s_{t+j})) \Vert=B_{\nabla ps},\\
&\sup_{s\in\mathcal{S}, a\in\mathcal{A}}\Vert \nabla_a p(s, a) \Vert=B_{\nabla pa},\\
&\sup_{s \in \mathcal{S}}\Vert \nabla_{s} r(s, \pi_{\theta}(s)) \Vert = B_{\nabla rs},\\
&\sup_{s \in \mathcal{S}, a \in \mathcal{A}}\Vert \nabla_{a} r(s, a) \Vert = B_{\nabla ra},\\
&\sup_{s \in \mathcal{S}}\Vert \nabla_s q^{\pi_{\theta}}(s,\pi_{\theta}(s)) \Vert = B_{\nabla qs},\\
&\sup_{s \in \mathcal{S}, a \in \mathcal{A}}\Vert \nabla_s q^{\pi_{\theta}}(s, a) \Vert = B_{\nabla qa},\\
&\sup_{s \in \mathcal{S}}\Vert \nabla_s Q_\omega(s,\pi_{\theta}(s)) \Vert = B_{\nabla Qs},\\
&\sup_{s \in \mathcal{S}, a \in \mathcal{A}}\Vert \nabla_a Q_\omega(s, a) \Vert = B_{\nabla Qa},\\
&\sup_{s \in \mathcal{S}}\Vert \nabla_s q^{\pi_{\theta}, \epsilon}(s,\pi_{\theta}(s)) \Vert = B_{\nabla q^{\epsilon}s},\\
&\sup_{s \in \mathcal{S}, a \in \mathcal{A}}\Vert \nabla_a q^{\pi_{\theta}, \epsilon}(s,a) \Vert = B_{\nabla q^{\epsilon}a}&&
\end{flalign*}
\end{regular_conditions}

\subsection{Proof of Lemma \ref{lemma.data_error}}\label{appendix.proof_of_lemma_data_error}
Inspired by \cite{luo2018algorithmic}, we first decompose the difference of two Q-value functions into a telescoping sum.

Denote the environment dynamics under noise as $o$. Let $W_j$ be the cumulative reward when we use true dynamics $p$ for $j$ steps and then $o$ for the rest of the steps, that is,
\begin{equation*}
    W_j = \Exp_{\substack{\forall t \ge 0, a_t=a\\ \forall 0\le t < j, s_{t+1} \sim p(\cdot|s_t,a_t)\\ \forall t \ge j, s_{t+1} \sim o(\cdot|s_t,a_t)}} \left\{ \sum_{t=0}^{\infty}\gamma^t r_t \bigg| s_0=s  \right\},
\end{equation*}
then we have $W_0 = q^{\pi_{\theta},\epsilon}(s, a)$ and $W_{\infty} = q^{\pi_{\theta}}(s, a)$,
\begin{equation*}
    q^{\pi_{\theta}}(s, a) - q^{\pi_{\theta},\epsilon}(s, a) = \sum_{j=0}^{\infty} (W_{j+1}-W_j)
\end{equation*}
where
\begin{equation*}
\begin{aligned}
W_j &= R + \Ed{s_j\sim \pi_{\theta}, p}{\Ed{s_{j+1} \sim o}{\gamma^{j+1}q^{\pi_{\theta}, \epsilon}_{s_{j+1}}}}\\
W_{j+1} &= R + \Ed{s_j\sim \pi_{\theta}, p} {\Ed{s_{j+1} \sim p}{\gamma^{j+1}q^{\pi_{\theta}, \epsilon}_{s_{j+1}} }},
\end{aligned}
\end{equation*}
in which $R$ is expected accumulative reward of the first $j$ steps from policy $\pi_{\theta}$ and true dynamics $p$, $q^{\pi_{\theta}, \epsilon}_{s}:=q^{\pi_{\theta}, \epsilon}(s, \pi_{\theta}(s))$. In our case,
\begin{equation*}
\begin{aligned}
    q^{\pi_{\theta}}(s, a) - q^{\pi_{\theta},\epsilon}(s, a) &= \sum_{j=0}^{\infty} (W_{j+1}-W_j)\\
    &=\sum_{j=0}^{\infty}\gamma^{j+1}(q^{\pi_{\theta}, \epsilon}_{s_{j+1}} - \mathbb{E}_{\epsilon}\left\{ q^{\pi_{\theta}, \epsilon}_{s_{j+1}+\epsilon} \right\}),
\end{aligned}
\end{equation*}
so the upper bound can be derived as follows,

\begin{align}
    &\sup_{s\in\mathcal{S}}\Vert \nabla_s q^{\pi_{\theta}}(s, \pi_{\theta}(s)) - \nabla_s q^{\pi_{\theta},\epsilon}(s, \pi_{\theta}(s))\Vert \nonumber\\
    \le&\sum_{j=0}^{\infty}\gamma^{j+1}\Ed{\epsilon}{\sup_{s\in\mathcal{S}}\Vert\nabla_s q^{\pi_{\theta}, \epsilon}_{s_{j+1}} - \nabla_s q^{\pi_{\theta}, \epsilon}_{s_{j+1}+\epsilon}\Vert}\nonumber\\
    \le&\sum_{j=0}^{\infty}\gamma^{j+1}\Ed{\epsilon}{ L_{\nabla q^\epsilon}\Vert\epsilon\Vert B_{\nabla ps}}\label{eq.data_error_bound}\\
    =&\frac{\gamma}{1-\gamma}L_{\nabla q^\epsilon}B_{\nabla ps}\Ed{\epsilon}{\lrVert{\epsilon}}=o(\Ed{\epsilon}{\lrVert{\epsilon}})\nonumber,
\end{align}
where it requires the continuity condition of $\nabla_s q^{\pi_{\theta},\epsilon}$ and the boundedness condition of the gradient through the state chain.
\begin{align}
    &\sup_{s\in\mathcal{S}, a\in\mathcal{A}}\Vert \nabla_a q^{\pi_{\theta}}(s, a) - \nabla_a q^{\pi_{\theta},\epsilon}(s, a)\Vert\nonumber \\
    \le&\sum_{j=0}^{\infty}\gamma^{j+1}\Ed{\epsilon}{\sup_{s\in\mathcal{S}, a\in\mathcal{A}}\Vert\nabla_a q^{\pi_{\theta}, \epsilon}_{s_{j+1}} - \nabla_a q^{\pi_{\theta}, \epsilon}_{s_{j+1}+\epsilon}\Vert}\nonumber\\
    \le&\sum_{j=0}^{\infty}\gamma^{j+1}\Ed{\epsilon }{ L_{\nabla q^\epsilon}\Vert\epsilon\Vert B_{\nabla ps} B_{\nabla pa}}\label{eq.data_error_bound2}\\
    =&\frac{\gamma}{1-\gamma}L_{\nabla q^\epsilon}B_{\nabla ps}B_{\nabla pa}\Ed{\epsilon}{\Vert\epsilon\Vert}=o(\Ed{\epsilon}{\Vert\epsilon\Vert})\nonumber,
\end{align}
where it requires the boundedness of $p$'s derivative w.r.t. $a$.

\subsection{Proof of Lemma \ref{lemma.r_q_grad_upper_bound}}
\label{appendix.proof_of_lemma_r_q_grad_bound}
We denote $\Delta_{s, t+i}:=\sup_{s_t \in \mathcal{S},a_t \in \mathcal{A}} \Vert \hat{s}_{t+i} - s_{t+i} \Vert$, and $\Delta_{\nabla s, t+i}:=\sup_{s_t \in \mathcal{S},a_t \in \mathcal{A}}\Vert \nabla_{\hat{s}_{t+i}} \hat{s}_{t+i+1} - \nabla_{s_{t+i}} s_{t+i+1} \Vert$. From the regular conditions from Appendix \ref{appendix.regular_conditions}, we have
\begin{equation}
\nonumber
\begin{aligned}
    &\Delta_{\nabla s, t+i}\\
    \le& \sup_{s_t \in \mathcal{S},a_t \in \mathcal{A}}\Vert \nabla_{\hat{s}_{t+i}} \hat{s}_{t+i+1} - \nabla_{s_{t+i}} \hat{s}_{t+i+1}+\\
    &\qquad\qquad\qquad\qquad\qquad\nabla_{s_{t+i}} \hat{s}_{t+i+1} - \nabla_{s_{t+i}} s_{t+i+1} \Vert\\
    \le& L_{\nabla f}\sup_{s_t \in \mathcal{S},a_t \in \mathcal{A}}\Vert \hat{s}_{t+i} - s_{t+i} \Vert + B_{\nabla \delta s}\\
    =& L_{\nabla f} \Delta_{s, t+i} + B_{\nabla \delta s}, i\ge 0.
\end{aligned}
\end{equation}
Then, we will discuss the upper bound in three cases: $k\ge 2$, $k=1$ and $k=0$.

First, when $k\ge2$, the reward term bias is:
\begin{strip}
\begin{align*}
&\sup_{s_t\in \mathcal{S}, a_t\in \mathcal{A}}\Vert \nabla_{a_t}\hat{r}_{t+k} - \nabla_{a_t}r_{t+k}\Vert\\
=&\sup_{s_t\in \mathcal{S}, a_t\in \mathcal{A}}\Vert \nabla_{a_t}\hat{r}_{t+k} -\nabla_{s_{t+k}} r_{t+k} 
\prod_{j=k}^{2}{\nabla_{\hat{s}_{t+j-1}} \hat{s}_{t+j}}  \nabla_{a_t} \hat{s}_{t+1} +
\nabla_{s_{t+k}} r_{t+k} \prod_{j=k}^{2}{\nabla_{\hat{s}_{t+j-1}} \hat{s}_{t+j}}  \nabla_{a_t} \hat{s}_{t+1} -
\nabla_{a_t}r_{t+k} \Vert\\
\le& L_{\nabla r} \sup_{s_t\in \mathcal{S}, a_t\in \mathcal{A}}\Vert \hat{s}_{t+k} - s_{t+k} \Vert B_{\nabla fs} B_{\nabla fa} +
\sup_{s_t\in \mathcal{S}, a_t\in \mathcal{A}} B_{\nabla rs} \Vert\prod_{j=k}^{2}{\nabla_{\hat{s}_{t+j-1}} \hat{s}_{t+j}}  \nabla_{a_t} \hat{s}_{t+1} - \prod_{j=k}^{2}{\nabla_{s_{t+j-1}} s_{t+j}}  \nabla_{a_t} s_{t+1} \Vert\\ 
=& L_{\nabla r} \Delta_{s, t+k} B_{\nabla fs} B_{\nabla fa} 
+\sup_{s_t\in \mathcal{S}, a_t\in \mathcal{A}} B_{\nabla rs}\Vert \prod_{j=k}^{2}{\nabla_{\hat{s}_{t+j-1}} \hat{s}_{t+j}}  \nabla_{a_t} \hat{s}_{t+1}-
\prod_{j=k}^{2}{\nabla_{\hat{s}_{t+j-1}} \hat{s}_{t+j}}  \nabla_{a_t} s_{t+1} +\\
&\ \ \ \ \ \ \ \ \ \ \ \ \ \ \ \ \ \ \ \ \ \ \ \ \ \ \ \ \ \ \ \ \ \ \ \ \ \ \ \ \ \ \ \ \ \ \ \ \ \ \ \ \ \ \ \ \ \ \ \ \ \ \ \ \ \ \ \ \ \ \ \ \prod_{j=k}^{2}{\nabla_{\hat{s}_{t+j-1}} \hat{s}_{t+j}}  \nabla_{a_t} s_{t+1}-
\prod_{j=k}^{2}{\nabla_{s_{t+j-1}} s_{t+j}}  \nabla_{a_t} s_{t+1} \Vert\\
\le& L_{\nabla r} \Delta_{s, t+k} B_{\nabla fs} B_{\nabla fa} + B_{\nabla rs}\bigg( B_{\nabla fs} B_{\nabla \delta a} +B_{\nabla pa} \Vert\prod_{j=k}^{2}{((\nabla_{\hat{s}_{t+j-1}} \hat{s}_{t+j}-\nabla_{s_{t+j-1}} s_{t+j})+\nabla_{s_{t+j-1}} s_{t+j})} 
-\prod_{j=k}^{2}{\nabla_{s_{t+j-1}} s_{t+j}} \Vert \bigg)\\
\le& L_{\nabla r} \Delta_{s, t+k} B_{\nabla fs} B_{\nabla fa} + B_{\nabla rs}\bigg( B_{\nabla fs} B_{\nabla \delta a} +B_{\nabla pa} \sum_{j=1}^{k-1}C_{k-1}^{j} \max_{1<i\le k} \Delta_{\nabla s, t+i-1}^{j} B_{\nabla ps}^{k-1-j} \bigg)\\
\le & L_{\nabla r} B_{\nabla fs} B_{\nabla fa} c_m k
+B_{\nabla rs} B_{\nabla fs} B_{\nabla \delta a} 
+B_{\nabla rs} B_{\nabla pa} \sum_{j=1}^{k-1}C_{k-1}^{j} B_{\nabla ps}^{k-1-j} (L_{\nabla f}c_m(k-1) + B_{\nabla \delta s})^j\\
\le & L_{\nabla r} B_{\nabla fs} B_{\nabla fa} c_m k
+B_{\nabla rs} B_{\nabla fs} B_{\nabla \delta a} 
+B_{\nabla rs} B_{\nabla pa} B_{\nabla ps}^{k-1} ((L_{\nabla f}c_m(k-1) + B_{\nabla \delta s})/B_{\nabla ps})^{(k-1)}\\
=& o((c_m k)^k)
\end{align*}
\end{strip}

\newpage

And the value term bias is:
\begin{align*}
&\sup_{s_t\in \mathcal{S}, a_t\in \mathcal{A}}\Vert \nabla_{a_t} \hat{Q}_{\omega,t+k}-\nabla_{a_t} q^{\pi_{\theta}}_{t+k}\Vert\\
\le &\sup_{s_t\in \mathcal{S}, a_t\in \mathcal{A}} \Vert  \nabla_{a_t} \hat{Q}_{\omega,t+k} -  \nabla_{a_t} Q_{\omega, t+k} +\\ &\qquad\qquad\qquad\qquad\qquad\qquad\nabla_{a_t} Q_{\omega, t+k} - \nabla_{a_t} q^{\pi_{\theta}}_{t+k} \Vert\\
\le & \sup_{s_t\in \mathcal{S}, a_t\in \mathcal{A}} \Vert  \nabla_{a_t} \hat{Q}_{\omega,t+k} -  \nabla_{a_t} Q_{\omega, t+k}\Vert+\\
&\qquad\sup_{s_t\in \mathcal{S}, a_t\in \mathcal{A}} \Vert \nabla_{s_{t+k}} Q_{\omega,t+k} - \nabla_{s_{t+k}} q^{\pi_{\theta}}_{t+k} \Vert B_{\nabla pa} B_{\nabla ps}\\
\le & o((c_m k)^k) + (\sup_{s_t\in \mathcal{S}, a_t\in \mathcal{A}} \Vert \nabla_{s_{t+k}} Q_{\omega,t+k} - \nabla_{s_{t+k}} q^{\pi_{\theta}, \epsilon}_{t+k} \Vert +\\
&\qquad\qquad\sup_{s_t\in \mathcal{S}, a_t\in \mathcal{A}} \Vert \nabla_{s_{t+k}} q^{\pi_{\theta}, \epsilon}_{t+k} - \nabla_{s_{t+k}} q^{\pi_{\theta}}_{t+k} \Vert) B_{\nabla pa} B_{\nabla ps}\\
\le & o((c_m k)^k)+B_{qs} B_{\nabla pa} B_{\nabla ps}+o(\Ed{\epsilon}{\Vert\epsilon\Vert}) B_{\nabla pa} B_{\nabla ps}\\
=& o((c_m k)^k+B_{qs} +\Ed{\epsilon}{\Vert\epsilon\Vert})
\end{align*}

Second, when $k=1$, the reward term bias is:
\begin{equation*}
\begin{aligned}
&\sup_{s_t\in \mathcal{S}, a_t\in \mathcal{A}}\Vert \nabla_{a_t}\hat{r}_{t+1} - \nabla_{a_t}r_{t+1}\Vert\\
=& \sup_{s_t\in \mathcal{S}, a_t\in \mathcal{A}} \Vert \nabla_{\hat{s}_{t+1}} \hat{r}_{t+1}\nabla_{a_t}\hat{s}_{t+1}
- \nabla_{s_{t+1}} r_{t+1}\nabla_{a_t}s_{t+1} \Vert\\
=& \sup_{s_t\in \mathcal{S}, a_t\in \mathcal{A}}
\Vert \nabla_{\hat{s}_{t+1}} \hat{r}_{t+1}\nabla_{a_t}\hat{s}_{t+1}
-\nabla_{s_{t+1}} r_{t+1}\nabla_{a_t}\hat{s}_{t+1}+\\
&\qquad\qquad\qquad\nabla_{s_{t+1}} r_{t+1}\nabla_{a_t}\hat{s}_{t+1}
- \nabla_{s_{t+1}} r_{t+1}\nabla_{a_t}s_{t+1} \Vert\\
\le& L_{\nabla r} \sup_{s_t\in \mathcal{S}, a_t\in \mathcal{A}}\Vert \hat{s}_{t+1} - s_{t+1} \Vert  B_{\nabla fa} + B_{\nabla rs}B_{\nabla \delta a}\\
\le&  L_{\nabla r}B_{\nabla fa}c_m + B_{\nabla rs}B_{\nabla \delta a}=o(c_m)
\end{aligned}
\end{equation*}

And the value term bias is:
\begin{equation*}
\begin{aligned}
&\sup_{s_t\in \mathcal{S}, a_t\in \mathcal{A}}\Vert \nabla_{a_t} \hat{Q}_{\omega,t+1}-\nabla_{a_t} q^{\pi_{\theta}}_{t+1}\Vert\\
\le &\sup_{s_t\in \mathcal{S}, a_t\in \mathcal{A}} \Vert  \nabla_{a_t} \hat{Q}_{\omega,t+1} -  \nabla_{a_t} Q_{\omega, t+1} +\\
&\qquad\qquad\qquad\qquad\qquad\nabla_{a_t} Q_{\omega, t+1} - \nabla_{a_t} q^{\pi_{\theta}}_{t+1} \Vert\\
\le&  L_{\nabla Q}B_{\nabla fa}c_m + B_{\nabla Qs}B_{\nabla \delta a} +\\
&\qquad\qquad\sup_{s_t\in \mathcal{S}, a_t\in \mathcal{A}} \Vert\nabla_{a_t} Q_{\omega, t+1} - \nabla_{a_t} q^{\pi_{\theta},\epsilon}_{t+1} \Vert
+\\
&\qquad\qquad\qquad\qquad\sup_{s_t\in \mathcal{S}, a_t\in \mathcal{A}} \Vert\nabla_{a_t} q^{\pi_{\theta},\epsilon}_{t+1} - \nabla_{a_t} q^{\pi_{\theta}}_{t+1} \Vert \\
\le&  L_{\nabla Q}B_{\nabla fa}c_m + B_{\nabla Qs}B_{\nabla \delta a} +B_{qs} B_{\nabla pa}+o(\Ed{\epsilon}{\Vert\epsilon\Vert})\\
=&o(c_m+B_{qs}+\Ed{\epsilon}{\Vert\epsilon\Vert})
\end{aligned}
\end{equation*}

Finally, when $k=0$, the reward term bias is:
\begin{equation*}
\begin{aligned}
\sup_{s_t\in \mathcal{S}, a_t\in \mathcal{A}}\Vert \nabla_{a_t}\hat{r}_t - \nabla_{a_t}r_t\Vert=0
\end{aligned}
\end{equation*}

And the value term bias is:
\begin{equation*}
\begin{aligned}
&\sup_{s_t\in \mathcal{S}, a_t\in \mathcal{A}}\Vert \nabla_{a_t} \hat{Q}_{\omega,t}-\nabla_{a_t} q^{\pi_{\theta}}_{t}\Vert\\
\le& \sup_{s_t\in \mathcal{S}, a_t\in \mathcal{A}}\Vert \nabla_{a_t} \hat{Q}_{\omega,t}-\nabla_{a_t} q^{\pi_{\theta},o}_{t}\Vert
+\\
&\qquad\qquad\qquad\qquad\sup_{s_t\in \mathcal{S}, a_t\in \mathcal{A}}\Vert \nabla_{a_t}q^{\pi_{\theta},\epsilon}_{t}-\nabla_{a_t} q^{\pi_{\theta}}_{t}\Vert\\
\le& o(B_{qa}+\Ed{\epsilon}{\Vert\epsilon\Vert}
\end{aligned}
\end{equation*}

\subsection{Proof of Theorem \ref{theorem.bias_upper_bound}}\label{appendix.proof_of_bias_upper_bound}
For simplicity, we denote
\begin{equation*}
\begin{aligned}
D_n(s_t, a_t)=\sum_{l=t}^{n-1+t} \gamma^{l-t} (\hat{r}_l- r_l)+ \gamma^{n}( \hat{Q}_{w,t+n}-q^{\pi_{\theta}}_{t+n}),
\end{aligned}
\end{equation*}
then
\begin{equation*}
\begin{aligned}
&\sup_{s_t\in\mathcal{S}}\Vert\nabla_{a_t}D_n(s_t, a_t)|_{a_t=\pi_{\theta}(s_t)}\Vert\\
=&\Vert\sum_{l=t}^{n-1+t} \gamma^{l-t} (\nabla_{a_t}\hat{r}_l- \nabla_{a_t}r_l)+\gamma^{n}( \nabla_{a_t}\hat{Q}_{\omega,t+n}-\nabla_{a_t}q^{\pi_{\theta}}_{t+n})\Vert\\
\le&\sum_{l=t}^{n-1+t} \gamma^{l-t} \sup\Vert\nabla_{a_t}\hat{r}_l- \nabla_{a_t}r_l\Vert+\\ &\qquad\qquad\qquad\gamma^{n}\sup\Vert\nabla_{a_t}\hat{Q}_{\omega,t+n}-\nabla_{a_t}q^{\pi_{\theta}}_{t+n}\Vert\\
\le&\left\{ 
    \begin{array}{lc}
        o(n(\gamma c_m n)^n + \gamma^n (B_{qs}+\Ed{\epsilon}{\Vert\epsilon\Vert}), & n \ge 1\\
        o(B_{qa}+\Ed{\epsilon}{\Vert\epsilon\Vert}), & n=0\\
    \end{array}\right. \\
\end{aligned}
\end{equation*}
Finally,
\begin{equation}
\nonumber
\begin{aligned}
&\Vert \nabla_{\theta}J_n(\theta) - \nabla_{\theta}J(\theta) \Vert\\
\le &\Ed{s_t}{\bigg\Vert \nabla_{a_t}D_n(s_t, a_t)\bigg|_{a_t=\pi_{\theta}(s_t)}\nabla_{\theta}\pi_{\theta}(s_t)\bigg\Vert}\\
\le &\Ed{s_t}{\bigg\Vert \nabla_{a_t}D_n(s_t, a_t)\bigg|_{a_t=\pi_{\theta}(s_t)}\bigg\Vert\bigg\Vert\nabla_{\theta}\pi_{\theta}(s_t)\bigg\Vert}\\
\le& B_{\nabla \theta}\sup_{s_t\in\mathcal{S}}\Vert\nabla_{a_t}D_n(s_t, a_t)|_{a_t=\pi_{\theta}(s_t)}\Vert\\
\le&\left\{ 
    \begin{array}{lc}
        o(n(\gamma c_m n)^n + \gamma^n (B_{qs}+\Ed{\epsilon}{\Vert\epsilon\Vert}), & n \ge 1\\
        o(B_{qa}+\Ed{\epsilon}{\Vert\epsilon\Vert}), & n=0\\
    \end{array}\right. \\
\end{aligned}
\end{equation}

\end{document}